\documentclass{article}
\usepackage{url}
\pdfoutput=1

\usepackage{authblk}

\usepackage{fullpage}

\usepackage{times}
\usepackage{graphicx} 
\usepackage{subfigure} 

\usepackage{multicol}
\usepackage{booktabs}
\usepackage[all]{xy}
\usepackage{tabularx}
\usepackage{setspace}

\usepackage[square,comma,numbers,sort&compress,sectionbib]{natbib}

\usepackage{algorithm}
\usepackage{algorithmic}

\usepackage[normalem]{ulem}

\usepackage{amsmath,amsthm,amssymb}

\usepackage{color}
\definecolor{Blue}{rgb}{0.9,0.3,0.3}
\usepackage{cancel}

\newcommand{\squishlist}{
   \begin{list}{$\bullet$}
    { \setlength{\itemsep}{0pt}      \setlength{\parsep}{3pt}
      \setlength{\topsep}{3pt}       \setlength{\partopsep}{0pt}
      \setlength{\leftmargin}{1.5em} \setlength{\labelwidth}{1em}
      \setlength{\labelsep}{0.5em} } }

\newcommand{\squishlisttwo}{
   \begin{list}{$\bullet$}
    { \setlength{\itemsep}{0pt}    \setlength{\parsep}{0pt}
      \setlength{\topsep}{0pt}     \setlength{\partopsep}{0pt}
      \setlength{\leftmargin}{2em} \setlength{\labelwidth}{1.5em}
      \setlength{\labelsep}{0.5em} } }

\newcommand{\squishend}{
    \end{list}  }

\newcommand{\myvec}[1]{\mathbf{#1}}

\newcommand{\vL}{\myvec{L}}

\newcommand{\vP}{\myvec{P}}

\newcommand{\vU}{\myvec{U}}

\newcommand{\vW}{\myvec{W}}

\newcommand{\be}{\begin{equation}}
\newcommand{\ee}{\end{equation}}
\newcommand{\bea}{\begin{eqnarray}}
\newcommand{\eea}{\end{eqnarray}}
\newcommand{\beaa}{\begin{eqnarray*}}
\newcommand{\eeaa}{\end{eqnarray*}}

\DeclareMathAlphabet{\mathpzc}{OT1}{pzc}{m}{n}

\usepackage{enumitem}

\usepackage{textcomp}

\usepackage[bottom]{footmisc}

\newcommand{\note}[1]{}
\renewcommand{\note}[1]{~\\\frame{\begin{minipage}[c]{0.48\textwidth}\vspace{2pt}\center{#1}\vspace{2pt}\end{minipage}}\vspace{3pt}\\}
\newcommand{\hide}[1]{}

\newcommand{\masrour}[1]{#1}

\newcommand{\cop}{\textup{cpld}}
\newcommand{\eps}{\epsilon}

\DeclareMathOperator{\argmax}{arg\,max}
\newtheorem{mydefinition}{Definition}
\newtheorem{proposition}[mydefinition]{Proposition}
\newtheorem{theorem}[mydefinition]{Theorem}
\newtheorem{remark}[mydefinition]{Remark}
\newtheorem{lemma}[mydefinition]{Lemma}
\newtheorem{corollary}[mydefinition]{Corollary}

\renewcommand{\algorithmicrequire}{\textbf{Input:}}
\renewcommand{\algorithmicensure}{\textbf{Return:}}
\renewcommand{\algorithmiccomment}[1]{// \textit{#1}}

\usepackage{hyperref}

\title{Copeland Dueling Bandits}

\author[1]{\mbox{Masrour Zoghi}}
\author[2]{\mbox{Zohar Karnin}}
\author[1]{\mbox{Shimon Whiteson}}
\author[1]{\mbox{Maarten de Rijke}}
\affil[1]{University of Amsterdam, Netherlands}
\affil[2]{Yahoo!\ Labs, Haifa, Israel}

\begin{document} 

\maketitle


\begin{abstract} 


A version of the dueling bandit problem is addressed in which a \emph{Condorcet winner} may not exist. Two algorithms are proposed that instead seek to minimize regret with respect to the \emph{Copeland winner}, which, unlike the Condorcet winner, is guaranteed to exist. The first, {\bf Copeland Confidence Bound (CCB)}, is designed for small numbers of arms, while the second, {\bf Scalable Copeland Bandits (SCB)}, works better for large-scale problems. We provide theoretical results bounding the regret accumulated by CCB and SCB, both substantially improving existing results.  Such existing results either offer bounds of the form $\mathcal{O}(K\log T)$ but require restrictive assumptions, or offer bounds of the form $\mathcal{O}(K^2\log T)$ without requiring such assumptions.  Our results offer the best of both worlds: $\mathcal{O}(K\log T)$ bounds without restrictive assumptions.


\end{abstract} 



\section{Introduction}
\label{sec:introduction}

\vspace{-4mm}

The \emph{dueling bandit problem} \citep{yue12:k-armed} arises naturally in domains where feedback is more reliable when given as a pairwise preference (e.g., when it is provided by a human) and specifying real-valued feedback instead would be arbitrary or inefficient. Examples include \emph{ranker evaluation} \citep{joachims2002:optimizing,YueJoachims:2011,hofmann:irj13} in information retrieval, ad placement and recommender systems. As with other \emph{preference learning} problems \citep{furnkranz2010}, feedback consists of a pairwise preference between a selected pair of arms, instead of scalar reward for a single selected arm, as in the $K$-armed bandit problem.

Most existing algorithms for the dueling bandit problem require the existence of a Condorcet winner, which is an arm that beats every other arm with probability greater than $0.5$. If such algorithms are applied when no Condorcet winner exists, no decision may be reached even after many comparisons.  This is a key weakness limiting their practical applicability. For example, in industrial ranker evaluation~\cite{schuth:multilevel:2014}, when many rankers must be compared, each comparison corresponds to a costly live experiment and thus the potential for failure if no Condorcet winner exists is unacceptable \cite{li:toward:2015}.

This risk is not merely theoretical.  On the contrary, recent experiments on $K$-armed dueling bandit problems based on information retrieval datasets show that dueling bandit problems without Condorcet winners arise regularly in practice \citep[Figure 1]{RCS2014}. In addition, we show in Appendix \ref{sec:condorcet} that there are realistic situations in ranker evaluation in information retrieval in which the probability that the Condorcet assumption holds decreases rapidly as the number of arms grows.  Since the $K$-armed dueling bandit methods mentioned above do not provide regret bounds in the absence of a Condorcet winner, applying them remains risky in practice. Indeed, we demonstrate empirically the danger of applying such algorithms to dueling bandit problems that do not have a Condorcet winner (cf. Appendix \ref{sec:experiments}).

\masrour{The non-existence of the Condorcet winner has been investigated extensively in social choice theory, where numerous definitions have been proposed, without a clear contender for the most suitable resolution \cite{schulze11:monotonic}. In the dueling bandit context, a few methods have been proposed to address this issue, e.g., SAVAGE \cite{Urvoy:2013}, PBR \cite{Busa-Fekete:2013} and RankEl \cite{Busa-Fekete:2014}, which use some of the notions proposed by social choice theorists, such as the Copeland score or the Borda score to measure the quality of each arm, hence determining what constitutes the best arm (or more generally the top-$k$ arms). In this paper, we focus on finding Copeland winners, which are arms that beat the greatest number of other arms, because it is a natural, conceptually simple extension of the Condorcet winner.}


Unfortunately, the methods mentioned above come with bounds of the form $\mathcal{O}(K^2\log T)$. In this paper, we propose two new $K$-armed dueling bandit algorithms for the Copeland setting with significantly improved bounds.  

The first algorithm, called {\bf Copeland Confidence Bound (CCB)}, is inspired by the recently proposed Relative Upper Confidence Bound method \cite{RUCB2014}, but modified and extended to address the unique challenges that arise when no Condorcet winner exists. We prove anytime high-probability and expected regret bounds for CCB of the form $\mathcal{O}(K^2 + K\log T)$. \masrour{Furthermore, the denominator of this result has much better dependence on the ``gaps'' arising from the dueling bandit problem than most existing results (cf. Sections \ref{sec:related-work} and \ref{sec:CCBtheory} for the details).} 

However, a remaining weakness of CCB is the additive $\mathcal{O}(K^2)$ term in its regret bounds.  In applications with large $K$, this term can dominate for any experiment of reasonable duration. For example, at Bing, 200 experiments are run concurrently on any given day \cite{kohavi:2013}, in which case the duration of the experiment needs to be longer than the age of the universe in nanoseconds before $K\log T$ becomes significant in comparison to $K^2$.

Our second algorithm, called {\bf Scalable Copeland Bandits (SCB)}, addresses this weakness by eliminating the $\mathcal{O}(K^2)$ term, achieving an expected regret bound of the form $\mathcal{O}(K \log K \log T)$. The price of SCB's tighter regret bounds is that, when two suboptimal arms are close to evenly matched, it may waste comparisons trying to determine which one wins in expectation.  By contrast, CCB can identify that this determination is unnecessary, yielding better performance unless there are very many arms.  CCB and SCB are thus complementary algorithms for finding Copeland winners.



Our main contributions are as follows:
\begin{enumerate}[leftmargin=*,topsep=0mm,parsep=0pt,itemsep=0pt,partopsep=0pt]
\item We propose two new algorithms that address the dueling bandit problem in the absence of a Condorcet winner, one designed for problems with small numbers of arms and the other scaling well with the number of arms.
\item We provide regret bounds that bridge the gap between two groups of results: those of the form $\mathcal{O}(K\log T)$ that make the Condorcet assumption, and those of the form $\mathcal{O}(K^2\log T)$ that do not make the Condorcet assumption. Our bounds are similar to those of the former but are as broadly applicable as the latter. Furthermore, the result for CCB has substantially better dependence on the gaps than the second group of results. 
\end{enumerate}

\masrour{In addition, Appendix \ref{sec:experiments} presents the results of an empirical evaluation of CCB and SCB using a real-life problem arising from information retrieval (IR). The experimental results mirror the theoretical ones.}





\vspace{-2mm}

\section{Problem Setting}
\label{sec:problem-setting}

\vspace{-4mm}

Let $K\geq2$. The \emph{$K$-armed dueling bandit} problem \cite{yue12:k-armed} is a modification of the \emph{$K$-armed bandit} problem \cite{thompson1933likelihood}. The latter considers $K$ arms $\{a_1,\ldots,a_K\}$ and at each \emph{time-step}, an arm $a_i$ can be \emph{pulled}, generating a \emph{reward} drawn from an unknown stationary distribution with expected value $\mu_i$. The $K$-armed \emph{dueling} bandit problem is a variation in which, instead of pulling a single arm, we choose a pair $(a_i,a_j)$ and receive one of them as the better choice, with the probability of $a_i$ being picked equal to an unknown constant $p_{ij}$ and that of $a_j$ being picked equal to $p_{ji}=1-p_{ij}$. A problem instance is fully specified by a \emph{preference matrix} $\vP=\left[p_{ij}\right]$, whose $ij$ entry is equal to $p_{ij}$.

Most previous work assumes the existence of a \emph{Condorcet winner} \cite{Urvoy:2013}: an arm, which without loss of generality we label $a_1$, such that $p_{1i} > \frac{1}{2}$ for all $i>1$. In such work, regret is defined relative to the Condorcet winner. However, Condorcet winners do not always exist \cite{RCS2014,RUCB2014}. In this paper, we consider a formulation of the problem that does not assume the existence of a Condorcet winner. 

Instead, we consider the \emph{Copeland dueling bandit problem}, which defines regret with respect to a \emph{Copeland winner}, which is an arm with maximal \emph{Copeland score}.  The Copeland score of $a_i$, denoted $\textup{Cpld}(a_i)$, is the number of arms $a_j$ for which  $p_{ij} > 0.5$.  The \emph{normalized Copeland score}, denoted $\cop(a_i)$, is simply $\frac{\textup{Cpld}(a_i)}{K-1}$. Without loss of generality, we assume that $a_1, \ldots, a_C$ are the Copeland winners, where $C$ is the number of Copeland winners.  We define regret as follows:

\begin{mydefinition}\label{def:regret} The {\bf regret} incurred by comparing $a_i$ and $a_j$ is $2\cop(a_1)-\cop(a_i)-\cop(a_j)$.
\end{mydefinition}


\begin{remark}  Since our results (see \S\ref{sec:theory}) establish bounds on the number of queries to non-Copeland winners, they can also be applied to other notions of regret. 
\end{remark}





\vspace{-2mm}

\section{Related Work}
\label{sec:related-work}

\vspace{-4mm}

Numerous methods have been proposed for the $K$-armed dueling bandit problem, including Interleaved Filter \cite{yue12:k-armed}, Beat the Mean \cite{YueJoachims:2011}, Relative Confidence Sampling \cite{RCS2014}, Relative Upper Confidence Bound (RUCB) \cite{RUCB2014}, Doubler and MultiSBM \cite{Ailon:2014}, and mergeRUCB \cite{mergeRUCB2015}, all of which require the existence of a Condorcet winner, and often come with bounds of the form $\mathcal{O}(K\log T)$. 
However, as observed in \cite{RUCB2014} and Appendix \ref{sec:condorcet}, real-world problems do not always have Condorcet winners.

There is another group of algorithms that do not assume the existence of a Condorcet winner, but have bounds of the form $\mathcal{O}(K^2 \log T)$ in the Copeland setting: Sensitivity Analysis of VAriables for Generic Exploration (SAVAGE) \cite{Urvoy:2013}, Preference-Based Racing (PBR) \cite{Busa-Fekete:2013} and Rank Elicitation (RankEl) \cite{Busa-Fekete:2014}. All three of these algorithms are designed to solve more general or more difficult problems, and they solve the Copeland dueling bandit problem as a special case. 

This work bridges the gap between these two groups by providing algorithms that are as broadly applicable as the second group but have regret bounds comparable to those of the first group. \masrour{Furthermore, in the case of the results for CCB, rather than depending on the smallest gap between arms $a_i$ and $a_j$, $\Delta_{\min}\hspace{-1mm}:=\hspace{-1mm}\min_{i>j}|p_{ij}-0.5|$, as in the case of many results in the Copeland setting,%
\footnote{Cf. \cite[Equation 9 in \S4.1.1]{Urvoy:2013} and \cite[Theorem 1]{Busa-Fekete:2013}.}
our regret bounds depend on a larger quantity that results in a substantially lower upper-bound, cf. \S\ref{sec:CCBtheory}.} 


In addition to the above, bounds have been proven for other notions of winners, including Borda \cite{Urvoy:2013,Busa-Fekete:2013,Busa-Fekete:2014}, Random Walk \cite{Busa-Fekete:2013,Negahban:2012}, and very recently von Neumann \cite{CDB:2015}. The dichotomy discussed also persists in the case of these results, which either rely on restrictive assumptions to obtain a linear dependence on $K$ or are more broadly applicable, at the expense of a quadratic dependence on $K$. A natural question for future work is whether the improvements achieved in this paper in the case of the Copeland winner can be obtained in the case of these other notions as well.

A related setting is that of \emph{partial monitoring games} \cite{piccolboni2001discrete}. 
While a dueling bandit problem can be modeled as a partial monitoring problem, doing so yields weaker results. 
In \cite{bartok2012adaptive}, the authors present problem-dependent bounds from which a regret bound of the form $\mathcal{O}(K^2\log T)$ can be deduced for the dueling bandit problem, whereas our work achieves a linear dependence in $K$.


\vspace{-2mm}

\section{Method}
\label{sec:method}

\vspace{-4mm}

We now present two algorithms that find Copeland winners.

\vspace{-2mm}

\subsection{Copeland Confidence Bound (CCB)}
\label{sec:CCB}

\vspace{-3mm}

CCB (see Algorithm \ref{alg:CCB}) is based on the principle of \emph{optimism followed by pessimism}: it maintains optimistic and pessimistic estimates of the preference matrix, i.e., matrices $\vU$ and $\vL$ (Line 6). It uses $\vU$ to choose an \emph{optimistic Copeland winner} $a_c$ (Lines 7--9 and 11--12), i.e., an arm that has some chance of being a Copeland winner.  Then, it uses $\vL$ to choose an \emph{opponent} $a_d$ (Line 13), i.e., an arm deemed likely to discredit the hypothesis that $a_c$ is indeed a Copeland winner.

More precisely, an optimistic estimate of the Copeland score of each arm $a_i$ is calculated using $\vU$ (Line 7), and $a_c$ is selected from the set of top scorers, with preference given to those in a shortlist, ${\cal B}_t$ (Line 11). Theses are arms that have, roughly speaking, been optimistic winners throughout  history. To maintain $\mathcal{B}_t$, as soon as CCB discovers that the optimistic Copeland score of an arm is lower than the pessimistic Copeland score of another arm, it purges the former from $\mathcal{B}_t$ (Line 9B).

The mechanism for choosing the opponent $a_d$ is as follows. The matrices $\vU$ and $\vL$ define a confidence interval around $p_{ij}$ for each $i$ and $j$. In relation to $a_c$, there are three types of arms: (1) arms $a_j$ s.t.\ the confidence region of $p_{cj}$ is strictly above $0.5$, (2) arms $a_j$ s.t.\ the confidence region of $p_{cj}$ is strictly below $0.5$, and (3) arms $a_j$ s.t.\ the confidence region of $p_{cj}$ contains $0.5$. Note that an arm of type (1) or (2) at time $t'$ may become an arm of type (3) at time $t>t'$ even without queries to the corresponding pair as the size of the confidence intervals increases as time goes on.

CCB always chooses $a_d$ from arms of type (3) because comparing $a_c$ and a type (3) arm is most informative about the Copeland score of $a_c$.
Among arms of type (3), CCB favors those that have confidently beaten arm $a_c$ in the past (Line 13), i.e., arms that in some round $t'<t$ were of type (2). Such arms are maintained in a shortlist of ``formidable'' opponents ($\mathcal{B}^i_t$) that are likely to confirm that $a_i$ is not a Copeland winner; these arms are favored when selecting $a_d$ (Lines 10 and 13).

The sets $\mathcal{B}^i_t$ are what speeds up the elimination of non-Copeland winners, enabling regret bounds that scale asymptotically with $K$ rather than $K^2$. Specifically, for a non-Copeland winner $a_i$, the set $\mathcal{B}^i_t$ will eventually contain $L_C\hspace{-.5mm}+\hspace{-.5mm}1$ strong opponents for $a_i$ (Line \ref{lineBs}C), where $L_C$ is the number of losses of each Copeland winner. Since $L_C$ is typically small (cf. \hspace{-1mm}Appendix \ref{sec:C-L_C}), asymptotically this leads to a bound of only $\mathcal{O}(\log T)$ on the number of time-steps when $a_i$ is chosen as an optimistic Copeland winner, instead of a bound of $\mathcal{O}(K\log T)$, which a more naive algorithm would produce.

\setlength{\textfloatsep}{2mm}
\begin{algorithm}[!t]
\begin{algorithmic}[1]
\REQUIRE A Copeland dueling bandit problem and an exploration parameter $\alpha > \frac{1}{2}$. 

\STATE $\vW = \left[w_{ij}\right] \gets \mathbf{0}_{K \times K} \; $ // 2D array of wins: $w_{ij}$ is the number of times $a_i$ beat $a_j$

\STATE $\mathcal{B}_1 = \{a_1, \ldots, a_K\}$ // potential best arms

\STATE $\mathcal{B}^i_1 = \varnothing$ for each $i=1,\ldots,K$ // potential to beat $a_i$

\STATE $\overline{L}_C = K$ // estimated max losses of a Copeland winner

\FOR{$t=1,2,\dots$}
   
   \STATE $\vU\hspace{-.5mm}:=\hspace{-.5mm}\left[u_{ij}\right]\hspace{-.5mm}=\hspace{-.5mm}\frac{\vW}{\vW+\vW^T}\hspace{-.5mm}+\hspace{-.5mm}\sqrt{\frac{\alpha\ln t}{\vW+\vW^T}}$ and $\vL\hspace{-.5mm}:=\hspace{-.5mm}\left[l_{ij}\right]\hspace{-.5mm}=\hspace{-.5mm}\frac{\vW}{\vW+\vW^T}\hspace{-.5mm}-\hspace{-.5mm}\sqrt{\frac{\alpha\ln t}{\vW+\vW^T}}$, with $u_{ii}\hspace{-.5mm}=\hspace{-.5mm}l_{ii}\hspace{-.5mm}=\hspace{-.5mm}\frac{1}{2}$, $\forall i$ 
   
   
   \STATE $\overline{\textup{Cpld}}(a_i) = \#\left\{k \,|\, u_{ik} \geq \frac{1}{2}, k \neq i\right\}$ 
   and $\underline{\textup{Cpld}}(a_i) = \#\left\{k \,|\, l_{ik} \geq \frac{1}{2}, k \neq i\right\}$ 
   
   \STATE $\mathcal{C}_t = \{a_i \,|\, \overline{\textup{Cpld}}(a_i) = \max_j \overline{\textup{Cpld}}(a_j) \}$
   
   \STATE Set $\mathcal{B}_t \gets \mathcal{B}_{t-1}$ and $\mathcal{B}^i_t \gets \mathcal{B}^i_{t-1}$ and update as follows: \label{lineBs}
   \begin{description}[leftmargin=3mm,topsep=0mm,parsep=0pt,itemsep=0pt,partopsep=0pt]
      \item[A. Reset disproven hypotheses:] \hspace{-5mm} If for any $i$ and $a_j\in\mathcal{B}^i_t$ we have $l_{ij}>0.5$, reset $\mathcal{B}_t$, $\overline{L}_C$ and $\mathcal{B}^k_t$ for all $k$ (i.e. set them to their original values as in Lines 2--4 above).
      \item[B. Remove non-Copeland winners:] \hspace{-5mm} For each $a_i\in\mathcal{B}_t$, if $\overline{\textup{Cpld}}(a_i)<\underline{\textup{Cpld}}(a_j)$ holds for any $j$, set $\mathcal{B}_t\gets\mathcal{B}_t\setminus\{a_i\}$, and if $|\mathcal{B}^i_t|\neq\overline{L}_C+1$, then set $\mathcal{B}^i_t\gets\{a_k|u_{ik} < 0.5 \}$. However, if $\mathcal{B}_t = \varnothing$, reset $\mathcal{B}_t$, $\overline{L}_C$ and $\mathcal{B}^k_t$ for all $k$.
      \item[C. Add Copeland winners:] \hspace{-5mm} For any $a_i\in\mathcal{C}_t$ with $\overline{\textup{Cpld}}(a_i)=\underline{\textup{Cpld}}(a_i)$, set $\mathcal{B}_t \gets \mathcal{B}_t \cup \{a_i\}$, $\mathcal{B}^i_t \gets \varnothing$ and $\overline{L}_C \gets K-1-\overline{\textup{Cpld}}(a_i)$. For each $j\neq i$, if we have $|\mathcal{B}^j_t|<\overline{L}_C+1$, set $\mathcal{B}^j_t\hspace{-.5mm}\gets\hspace{-.5mm}\varnothing$, and if $|\mathcal{B}^j_t|\hspace{-.5mm}>\hspace{-.5mm}\overline{L}_C\hspace{-.5mm}+\hspace{-.5mm}1$, randomly choose $\overline{L}_C\hspace{-.5mm}+\hspace{-.5mm}1$ elements of $\mathcal{B}^j_t$ and remove the rest.
   \end{description}
   
   \STATE With probability $1/4$, sample $(c,d)$ uniformly from the set $\{(i,j)~|~a_j\in\mathcal{B}^i_t \textup{ and } 0.5\in[l_{ij},u_{ij}]\}$ (if it is non-empty) and skip to Line 14.

   \STATE If $\mathcal{B}_t \cap \mathcal{C}_t \neq \varnothing$, then with probability $2/3$, set $\mathcal{C}_t \gets \mathcal{B}_t \cap \mathcal{C}_t$.
   
   \STATE Sample $a_c$ from $\mathcal{C}_t$ uniformly at random. 

   \STATE \raggedright  With probability $1/2$, choose the set $\mathcal{B}^i$ to be either $\mathcal{B}^i_t$ or $\{a_1,\ldots,a_K\}$ and then set $d~\gets~\displaystyle\argmax_{\{j \in \mathcal{B}^i \,|\,l_{jc} \leq 0.5\}} u_{jc}$. If there is a tie, $d$ is not allowed to be equal to $c$.

   \STATE Compare arms $a_c$ and $a_d$ and increment $w_{cd}$ or $w_{dc}$ depending on which arm wins.
\ENDFOR
\end{algorithmic}
\caption{Copeland Confidence Bound}
\label{alg:CCB}
\end{algorithm}

\begin{algorithm}[b]
\caption{Approximate Copeland Bandit Solver}
\label{alg:copeland-apx}
\begin{algorithmic}[1]
{
\REQUIRE A Copeland dueling bandit problem with preference matrix $\vP = [p_{ij}]$, %
failure probability $\delta>0$, and approximation parameter $\eps>0$. Also, define $[K] := \{1,\ldots,K\}$.

\STATE Define a random variable $\text{reward}(i)$ for $i \in [K]$ as the following procedure: pick a uniformly random $j \neq i$ from $[K]$; query the pair $(a_i,a_j)$ sufficiently many times in order to determine w.p.\ at least $1-\delta/K^2$ whether $p_{ij}>1/2$; return $1$ if $p_{ij}>0.5$ and $0$ otherwise.

\STATE Invoke Algorithm~\ref{alg:kl-MAB}, where in each of its calls to $\text{reward}(i)$, the feedback is determined by the above stochastic process.

\ENSURE The same output returned by Algorithm~\ref{alg:kl-MAB}.
}
\end{algorithmic}
\end{algorithm}

\vspace{-2mm}

\subsection{Scalable Copeland Bandits (SCB)} \label{sec:scb}

\vspace{-3mm}

SCB is designed to handle dueling bandit problems with large numbers of arms.  It is based on an arm-identification algorithm, described in Algorithm~\ref{alg:copeland-apx}, designed for a PAC setting, i.e., it finds an $\eps$-Copeland winner with probability $1-\delta$, although we are primarily interested in the case with $\epsilon=0$.  Algorithm~\ref{alg:copeland-apx} relies on a reduction to a $K$-armed bandit problem where we have direct access to a noisy version of the Copeland score; the process of estimating the score of arm $a_i$ consists of comparing $a_i$ to a random arm $a_j$ until it becomes clear which arm beats the other.  
The sample complexity bound, which yields the regret bound, is achieved by combining a bound for $K$-armed bandits and a bound on the number of arms that can have a high Copeland score. 


Algorithm~\ref{alg:copeland-apx} calls a $K$-armed bandit algorithm as a subroutine.  To this end, we use the KL-based arm-elimination algorithm (a slight modification of Algorithm~2 in~\cite{cappe2013kullback}) described in Algorithm~\ref{alg:kl-MAB} in Appendix \ref{sec:KL analysis}. It implements an elimination tournament with confidence regions based on the KL-divergence between probability distributions. 

Combining this with the \emph{squaring trick}, a modification of the \emph{doubling trick} that reduces the number of partitions from $\log T$ to $\log \log T$, the SCB algorithm, described in Algorithm~\ref{alg:copeland-scal-reg}, repeatedly calls Algorithm~\ref{alg:copeland-apx} but force-terminates if an increasing threshold is reached.
 If it terminates early, then the identified arm is played against itself until the threshold is reached.

\vspace{-3mm}

\begin{algorithm}[h]
\caption{Scalable Copeland Bandits}
\label{alg:copeland-scal-reg}
\begin{algorithmic}[1]
{
\REQUIRE A Copeland dueling bandit problem with preference matrix $\vP = [p_{ij}]$

\FORALL{$r=1,2,\ldots$ } 
	\STATE Set $T = 2^{2^r} $ and run Algorithm~\ref{alg:copeland-apx} with failure probability $\log(T)/T$ in order to find an exact Copeland winner ($\eps=0$); force-terminate if it requires more than $T$ queries.
	
	\STATE Let $T_0$ be the number of queries used by invoking Algorithm~\ref{alg:copeland-apx}, and let $a_i$ be the arm produced by it; query the pair $(a_i,a_i)$ $T-T_0$ times. 
\ENDFOR
}
\end{algorithmic}
\end{algorithm}



\vspace{-2mm}

\section{Theoretical Results}
\label{sec:theory}

\vspace{-4mm}

In this section, we present regret bounds for both CCB and SCB. Assuming that the number of Copeland winners and the number of losses of each Copeland winner are bounded,\footnote{See Appendix \ref{sec:C-L_C} for experimental evidence that this is the case in practice.} CCB's regret bound takes the form $\mathcal{O}(K^2+K\log T)$, while SCB's is of the form $\mathcal{O}(K\log K \log T)$. Note that these bounds are not directly comparable. When there are relatively few arms, CCB is expected to perform better.  By contrast, when there are many arms SCB is expected to be superior. Appendix \ref{sec:experiments} provides empirical evidence to support these expectations.

Throughout this section we impose the following condition on the preference matrix:


\begin{itemize}[topsep=0mm,parsep=0pt,itemsep=0pt,partopsep=0pt]
\item[{\bf A}] There are no ties, i.e., for all pairs $(a_i,a_j)$ with $i\neq j$, we have $p_{ij} \neq 0.5$. 
\end{itemize} 


This assumption is not very restrictive in practice.  For example, in the ranker evaluation setting from information retrieval, each arm corresponds to a ranker, a complex and highly engineered system, so it is unlikely that two rankers are indistinguishable.  Furthermore, some of the results we present in this section actually hold under even weaker assumptions.  However, for the sake of clarity, we defer a discussion of these nuanced differences to  Appendix \ref{sec:counts-proof}.



\vspace{-2mm}

\subsection{Copeland Confidence Bounds (CCB)}
\label{sec:CCBtheory}

\vspace{-3mm}

To analyze Algorithm \ref{alg:CCB}, consider a $K$-armed Copeland bandit problem with arms $a_1,\ldots,a_K$ and preference matrix $\vP = [p_{ij}]$, such that arms $a_1,\ldots,a_C$ are the Copeland winners, with $C$ being the number of Copeland winners. Throughout this section, we assume that the parameter $\alpha$ in Algorithm \ref{alg:CCB} satisfies $\alpha\hspace{-.5mm}>\hspace{-.5mm}0.5$, unless otherwise stated. We first define the relevant quantities:

\begin{mydefinition}\label{def:deltas} Given the above setting we define:\footnote{See Tables \ref{tbl:notation1} and \ref{tbl:notation2} for a summary of the definitions used in this paper.}

\vspace{-2mm}

\begin{enumerate}[leftmargin=*,topsep=0mm,parsep=0pt,itemsep=0pt,partopsep=0pt]
\item $\mathcal{L}_i := \{a_j\, |\, p_{ij} < 0.5\}$, i.e., the arms to which $a_i$ loses, and $L_C := |\mathcal{L}_1|$. 
\item\label{item:Delta_min} $\Delta_{ij} := |p_{ij}-0.5|$ and $\Delta_{\min} := \min_{i\neq j} \Delta_{ij}$
\item\label{item:i*} Given $i > C$, define $i^*$ as the index of the $(L_C+1)^{th}$ largest element in the set $\{ \Delta_{ij} \,|\, p_{ij} < 0.5 \}$.
\item\label{item:Delta*i} Define $\Delta^*_i$ to be $\Delta_{ii^*}$ if $i>C$ and $0$ otherwise. Moreover, let us set $\Delta^*_{\min} := \min_{i>C} \Delta^*_i$.
\item\label{item:Delta*ij} Define $\Delta^*_{ij}$ to be $\Delta^*_i + \Delta_{ij}$ if $p_{ij} \geq 0.5$ and $\max\{\Delta^*_i,\Delta_{ij}\}$ otherwise.\footnote{See Figures \ref{fig:CopelandNonCopelandFigure} and \ref{fig:nonCopelandFigure} for a pictorial explanation.}
\item\label{item:Delta} $\Delta := \min\left\{\min_{i \leq C < j} \Delta_{ij}, \Delta^*_{\min} \right\}$, where $\Delta^*_{\min}$ is defined as in item \ref{item:Delta*i} above.
\item\label{item:Cdelta} $C(\delta) := \left((4\alpha-1)K^2/(2\alpha-1)\delta\right)^{\frac{1}{2\alpha-1}}$ where $\alpha$ is as in Algorithm \ref{alg:CCB}.
\item $N_{ij}^{\delta}(t)$ is the number of time-steps between times $C(\delta)$ and $t$ when $a_i$ was chosen as the optimistic Copeland winner and $a_j$ as the challenger.
Also, $\widehat{N}_{ij}^{\delta}(t)$ is defined to be $(4\alpha\ln t)/\left(\Delta^*_{ij}\right)^2$ if $i \neq j$, $0$ if $i=j>C$ and $t$ if $i=j\leq C$. 
We also define $\widehat{N}^{\delta}(t) := \sum_{i \neq j} \widehat{N}_{ij}^{\delta}(t)+1$.
\end{enumerate}
\end{mydefinition}


Using this notation, our expected regret bound for CCB takes the form:
%
%
$\mathcal{O}\left(\frac{K^2 + (C+L_C) K\ln T}{\Delta^2}\right)~~~\refstepcounter{equation}(\theequation)\label{eqn:ExpRegret}$

\vspace{-2mm}

This result is proven in two steps. First,  Proposition \ref{prop:counts} bounds the number of comparisons involving non-Copeland winners, yielding a result of the form $\mathcal{O}(K^2\ln T)$. Second, Theorem \ref{thm:KlogT} closes the gap between this bound and that of \eqref{eqn:ExpRegret} by showing that, beyond a certain time horizon, CCB selects non-Copeland winning arms as the optimistic Copeland winner very infrequently.

Note that we have $\Delta^*_{ij}~\geq~\Delta_{ij}$ for all pairs $i \neq j$. Thus, for simplicity, the analysis in this section can be read as if the bounds were given in terms of $\Delta_{ij}$. We use $\Delta^*_{ij}$ instead because it gives tighter upper bounds. In particular, simply using the gaps $\Delta_{ij}$ would replace the denominator of the expression in \eqref{eqn:ExpRegret} with $\Delta_{\min}^2$, which leads to a substantially worse regret bound in practice. For instance, in the ranker evaluation application used in the experiments, this change would  on average increase the regret bound by a factor that is of the order of tens of thousands. 
See Appendix \ref{sec:Delta} for a more quantitative discussion of this point.

We can now state our first bound, proved in Appendix \ref{sec:counts-proof} under weaker assumptions.


\begin{proposition}\label{prop:counts}
Given any $\delta > 0$ and $\alpha > 0.5$, if we apply CCB (Algorithm \ref{alg:CCB}) to a dueling bandit problem satisfying Assumption {\bf A}, the following holds with probability $1-\delta$: for any $T>C(\delta)$ and any pair of arms $a_i$ and $a_j$, we have $N_{ij}^{\delta}(T) \leq \widehat{N}_{ij}^{\delta}(T)$.  
\end{proposition}


One can sum the inequalities in the last proposition over pairs $(i,j)$ to get a regret bound of the form $\mathcal{O}(K^2 \log T)$ for Algorithm \ref{alg:CCB}. However, as Theorem \ref{thm:KlogT} will show, we can use the properties of the sets $\mathcal{B}^i_t$ to obtain a tighter regret bound of the form $\mathcal{O}(K\log T)$. Before stating that theorem, we need a few definitions and lemmas. We begin by defining the key quantity:

\begin{mydefinition}\label{def:Tdelta}
Given a preference matrix $\vP$ and $\delta > 0$, then $T_\delta$ is the smallest integer satisfying \\
$ T_\delta \hspace{.5mm} \geq \hspace{.5mm} C(\frac{\delta}{2}) \hspace{-.5mm} + \hspace{-.5mm} 8K^2(L_C \hspace{-.5mm} + \hspace{-.5mm} 1)^2\ln\hspace{-.5mm}\frac{6K^2}{\delta} \hspace{-.5mm} + \hspace{-.5mm} K^2\ln\hspace{-.5mm}\frac{6K}{\delta} \hspace{-.5mm} + \hspace{-.5mm} \frac{32\alpha K (L_C+1)}{\Delta^2_{\min}}\ln T_{\delta} \hspace{-.5mm} + \hspace{-.5mm} \widehat{N}^{\frac{\delta}{2}} \hspace{-.5mm} (T_{\delta}) \hspace{-.5mm} + \hspace{-.5mm} 4K \displaystyle\max_{i>C}\widehat{N}^{\frac{\delta}{2}}_i \hspace{-.5mm} (T_\delta)$.
\end{mydefinition}


\begin{remark}  $T_\delta$ is $\mathrm{poly}(K,\delta^{-1})$ and our regret bound below scales as $\log T_\delta$.
%
\end{remark}

The following two lemmas are key to the proof of Theorem \ref{thm:KlogT}. Lemma \ref{lem:SetsB} (proved in Appendix \ref{sec:SetsB-proof}) states that, with high probability by time $T_\delta$, each set $\mathcal{B}^i_t$ contains $L_C+1$ arms $a_j$, each of which beats $a_i$ (i.e., $p_{ij} < 0.5$). This fact then allows us to  prove Lemma \ref{lem:NumNonCopeland} (Appendix \ref{sec:NumNonCopeland-proof}), which states that, after time-step $T_\delta$, the rate of suboptimal comparisons is $\mathcal{O}(K\ln T)$ rather than $\mathcal{O}(K^2\ln T)$.

\begin{lemma}\label{lem:SetsB}
Given $\delta > 0$, with probability $1-\delta$, each set $\mathcal{B}^i_{T_\delta}$ with $i > C$ contains exactly $L_C+1$ elements with each element $a_j$ satisfying $p_{ij} < 0.5$. Moreover, for all $t \in [T_\delta,T]$, we have $\mathcal{B}^i_t = \mathcal{B}^i_{T_\delta}$.
\end{lemma}



\begin{lemma}\label{lem:NumNonCopeland}
Given a Copeland bandit problem satisfying Assumption {\bf A} and any $\delta > 0$, with probability $1-\delta$ the following holds: the number of time-steps between $T_{\delta/2}$ and $T$ when each non-Copeland winner $a_i$ can be chosen as optimistic Copeland winners (i.e., times when arm $a_c$ in Algorithm \ref{alg:CCB} satisfies $c>C$) is bounded by
$\widehat{N}^i := 2\widehat{N}^i_{\mathcal{B}} + 2\sqrt{\widehat{N}^i_{\mathcal{B}}} \ln \frac{2K}{\delta}$, where $\widehat{N}^i_{\mathcal{B}} := \hspace{-.5mm}\sum_{j \in \mathcal{B}^i_{T_{\delta/2}}}\hspace{-2mm} \widehat{N}_{ij}^{\delta/4}(T)$.
%
\end{lemma}


\begin{remark}
Due to Lemma \ref{lem:SetsB}, with high probability we have $\widehat{N}^i_{\mathcal{B}} \leq \frac{(L_C + 1)\ln T}{\left(\Delta^*_{\min}\right)^2}$ for each $i > C$ and so the total number of times between $T_\delta$ and $T$ when a non-Copeland winner is chosen as an optimistic Copeland winner is in $\mathcal{O}(KL_C\ln T)$ for a fixed minimal gap $\Delta^*_{\min}$. The only other way a suboptimal comparison can occur is if a Copeland winner is compared against a non-Copeland winner, and according to Proposition \ref{prop:counts}, the number of such occurrences is bounded by $\mathcal{O}(KC\ln T)$. Hence, the number of suboptimal comparisons is in $\mathcal{O}(K\ln T)$ assuming that $C$ and $L_C$ are bounded. In Appendix \ref{sec:C-L_C}, we provide experimental evidence for this.
\end{remark}


We now define the quantities needed to state the main theorem.

\begin{mydefinition}\label{def:CCBconst}
We define the following three quantities:
$A_\delta^{(1)} := C(\delta/4) + \widehat{N}^\delta(T_{\delta/2})$,
$A_\delta^{(2)} := \sum_{i > C} \frac{\sqrt{L_C+1}}{\Delta^*_i} \ln \frac{2K}{\delta}$ and 
$A^{(3)} := \sum_{i \leq C < j} \frac{1}{\left(\Delta_{ij}\right)^2} + 2\sum_{i > C} \frac{L_C+1}{\left(\Delta^*_i\right)^2}$.
\end{mydefinition}


\begin{theorem}\label{thm:KlogT}
Given a Copeland bandit problem satisfying Assumption {\bf A} and any $\delta > 0$ and $\alpha > 0.5$, with probability $1-\delta$, the regret accumulated by CCB is bounded by the following:

\vspace{-5mm}

\begin{align*}
A_\delta^{(1)} + A_\delta^{(2)} \sqrt{\ln T} + A^{(3)} \ln T \hspace{1mm} \leq \hspace{1mm} A_\delta^{(1)} + A_\delta^{(2)} \sqrt{\ln T} + \frac{2K(C+L_C+1)}{\Delta^2} \ln T.
\end{align*}
\end{theorem}

\vspace{-2mm}

For a general assessment of the above quantities, assuming that $L_C$ and $C$ are both $\mathcal{O}(1)$, 
the above quantities in terms of $K$ become $A_\delta^{(1)} = \mathcal{O}(K^2)$, $A_\delta^{(2)} = \mathcal{O}(K\log(K))$, $A^{(3)}=\mathcal{O}(K)$. Hence, the above bound boils down to the expression in \eqref{eqn:ExpRegret}.
We now turn to the proof of the theorem.

\begin{proof}[Proof of Theorem \ref{thm:KlogT}]
Let us consider the two disjoint time-intervals $[1,T_{\delta/2}]$ and $(T_{\delta/2},T]$:
\begin{description}[leftmargin=*,topsep=0mm,parsep=0pt,itemsep=1pt,partopsep=0pt]
\item[{[1,$\mathbf{T}_{\delta/2}$]}:] In this case, applying Proposition \ref{prop:counts} to $T_\delta$, we get that the number of time-steps when a non-Copeland winner was compared against another arm is bounded by $A_\delta^{(1)}$. As the maximum regret such a comparison can incur is $1$, this deals with the first term in the above expression.

\item[{($\mathbf{T}_{\delta/2}$,T]:}] In this case, applying Lemma \ref{lem:NumNonCopeland}, we get the other two terms in the above regret bound. \qedhere
\end{description}
\end{proof}

\vspace{-3mm}

Now that we have the high probability regret bound given in Theorem \ref{thm:KlogT}, we can deduce the expected regret result claimed in \eqref{eqn:ExpRegret} for $\alpha > 1$, as a corollary by integrating $\delta$ over the interval $[0,1]$.
%
%

\vspace{-2mm}

\subsection{Scalable Copeland Bandits}
\label{sec:SCBtheory}

\vspace{-3mm}

We now turn to our regret result for SCB, which lowers the $K^2$ dependence in the additive constant of CCB's regret result to $K \log K$. We begin by defining the relevant quantities:

\begin{mydefinition}\label{def:SCBdefs}
Given a $K$-armed Copeland bandit problem and an arm $a_i$, we define the following:
\begin{enumerate}[leftmargin=*,topsep=0mm,parsep=0pt,itemsep=0pt,partopsep=0pt]
\item Recall that $\cop(a_i) := \textup{Cpld}(a_i)/(K-1)$ is called the normalized Copeland score. 
\item\label{def:ep-copeland} $a_i$ is an $\eps$-Copeland-winner if $1-\cop(a_i) \leq \left(1- \cop(a_1) \right) (1+\eps)$.
\item $\Delta_i:=\max\{\cop(a_1)-\cop(a_{i}),1/(K-1)\}$ 
and $H_i := \sum_{j \neq i} \frac{1}{\Delta_{ij}^2}$, with $H_{\infty} := \max_i H_i$.
\item $\Delta_i^\eps  = \max \left\{ \Delta_i, \eps(1-\cop(a_1)) \right\}$.
\end{enumerate}
\end{mydefinition}

We now state our main \masrour{scalability} result:

\begin{theorem} \label{thm:scalable-main_m}
Given a Copeland bandit problem satisfying Assumption {\bf A}, the expected regret of SCB (Algorithm~\ref{alg:copeland-scal-reg}) is bounded by
%
%
${\cal O} \left(\frac{1}{K} \sum_{i=1}^K \frac{H_i  (1-\cop(a_{i})) }{ \Delta_i^2} \right) \log(T)$,
which in turn can be bounded by
${\cal O} \left(  \frac{K (L_C+\log K)\log T}{ \Delta_{\min}^2} \right)$, 
where $L_C$ and $\Delta_{\min}$ are as in Definition \ref{def:deltas}.


\end{theorem}



Recall that SCB is based on Algorithm \ref{alg:copeland-apx}, an arm-identification algorithm that identifies a Copeland winner with high probability.  As a result, Theorem~\ref{thm:scalable-main_m} is an immediate corollary of Lemma~\ref{lem:apx_cop_main_m}, obtained by using the well known squaring trick.  As mentioned in Section \ref{sec:scb}, the squaring trick is a minor variation on the doubling trick that reduces the number of partitions from $\log T$ to $\log \log T$.


Lemma~\ref{lem:apx_cop_main_m} is a result for finding an $\eps$-approximate Copeland winner (see Definition \ref{def:SCBdefs}.\ref{def:ep-copeland}). Note that, for the regret setting, we are only interested in the special case with $\eps=0$, i.e., the problem of identifying the best arm.

\begin{lemma}\label{lem:apx_cop_main_m}
With probability $1-\delta$, Algorithm~\ref{alg:copeland-apx} finds an $\eps$-approximate Copeland winner by time

\vspace{-6mm}

\[ {\cal O}\left(\frac{1}{K}\sum_{i=1}^K\frac{H_i(1-\cop(a_{i}))}{\left(\Delta_i^\eps\right)^2}\right)\log(1/\delta)\leq
{\cal O}\left(H_{\infty}\left(\log(K)+\min\left\{\eps^{-2},L_C\right\}\right)\right)\log(1/\delta). \]

\vspace{-4mm}

assuming\footnote{The exact expression requires replacing $\log(1/\delta)$ with $\log(KH_{\infty}/\delta)$.} $\delta =  (K H_{\infty})^{\Omega(1)}$.
In particular when there is a Condorcet winner ($\cop(a_{1})=1, L_C=0$) or more generally $\cop(a_{1}) = 1-\mathcal{O}(1/K), L_C=\mathcal{O}(1)$, an exact solution is found with probability at least $1-\delta$ by using an expected number of queries of at most
${\cal O} \left(  H_{\infty} (L_C+\log K) \right) \log(1/\delta)$.
\end{lemma}


In the remainder of this section, we sketch the main ideas underlying the proof of Lemma \ref{lem:apx_cop_main_m}, detailed in  Appendix \ref{sec:SCBAnalysis}. We first  treat the simpler deterministic setting in which a single query suffices to determine which of a pair of arms beats the other. While a solution can easily be obtained using $K(K-1)/2$ many queries, we aim for one with query complexity linear in $K$. The main ingredients of the proof are as follows:

\begin{enumerate}[leftmargin=*,topsep=0mm,parsep=0pt,itemsep=0pt,partopsep=0pt]
\item $\cop(a_i)$ is the mean of a Bernoulli random variable defined as such: sample uniformly at random an index $j$ from the set $\{1,\ldots,K\} \setminus \{i\}$ and return $1$ if $a_i$ beats $a_j$ and $0$ otherwise.
\item Applying a KL-divergence based arm-elimination algorithm (Algorithm \ref{alg:kl-MAB}) to the $K$-armed bandit arising from the above observation, we obtain a bound by dividing the arms into two groups: those with Copeland scores close to that of the Copeland winners, and the rest. For the former, we use the result from Lemma \ref{lem:copeland obs} to bound the number of such arms; for the latter,  the resulting regret is dealt with using Lemma \ref{lem:rhs0}, which exploits the possible distribution of Copeland scores.

\end{enumerate}

Let us state the two key lemmas here:

\begin{lemma} \label{lem:copeland obs}
Let $D \subset \{a_1,\ldots,a_K\}$ be the set of arms for which $\cop(a_{i}) \geq 1-d/(K-1)$, that is arms that are beaten by at most $d$ arms. Then $|D| \leq 2d+1$. 
\end{lemma}

\vspace{-5mm}

\begin{proof}
Consider a fully connected directed graph, whose node set is $D$ and the arc $(a_i,a_j)$ is in the graph if arm $a_{i}$ beats arm $a_{j}$. By the definition of $\cop$, the in-degree of any node $i$ is upper bounded by $d$. Therefore, the total number of arcs in the graph is at most $|D|d$. Now, the full connectivity of the graph implies that the total number of arcs in the graph is exactly $|D|(|D|-1)/2$. Thus, $|D|(|D|-1)/2 \leq |D|d$ and the claim follows. \qedhere

\end{proof}

\begin{lemma} \label{lem:rhs0}
The sum $\sum_{\{i | \cop(a_i) < 1\}} \frac{1}{1-\cop(a_i)}$ is in $\mathcal{O}(K\log K)$.
\end{lemma}

\vspace{-4mm}

\begin{proof}
Follows from Lemma \ref{lem:copeland obs} via a careful partitioning of arms. Details are in Appendix \ref{sec:SCBAnalysis}. 
\end{proof}
%

Given the structure of Algorithm \ref{alg:copeland-apx}, the stochastic case is similar to the deterministic case for the following reason: while the latter requires a single comparison between arms $a_i$ and $a_j$ to determine which arm beats the other, in the stochastic case, we need roughly $\frac{\log(K\log(\Delta_{ij}^{-1})/\delta)}{\Delta_{ij}^2}$ comparisons between the two arms to correctly answer the same question with probability at least $1-\delta/K^2$.

\vspace{-2mm}

\section{Conclusion}
\label{sec:conclusion}

\vspace{-4mm}

In many applications that involve learning from human behavior, feedback is more reliable when provided in the form of pairwise preferences. In the dueling bandit problem, the goal is to use such pairwise feedback to find the most desirable choice from a set of options. Most existing work in this area assumes the existence of a Condorcet winner, i.e., an arm that beats all other arms with probability greater than $0.5$. Even though these results have the advantage that the bounds they provide scale linearly in the number of arms, their main drawback is that in practice the Condorcet assumption is too restrictive. 
By contrast, other results that do not impose the Condorcet assumption achieve bounds that scale quadratically in the number of arms.

In this paper, we set out to solve a natural generalization of the problem, where instead of assuming the existence of a Condorcet winner, we seek to find a Copeland winner, which is guaranteed to exist. We proposed two algorithms to address this problem: one for small numbers of arms, called CCB; and a more scalable one, called SCB, that works better for problems with large numbers of arms. We provided theoretical results bounding the regret accumulated by each algorithm: these results improve substantially over existing results in the literature, by filling the gap that exists in the current results, namely the discrepancy between results that make the Condorcet assumption and are of the form $\mathcal{O}(K\log T)$ and the more general results that are of the form $\mathcal{O}(K^2\log T)$. 

Moreover, we have included empirical results on both a dueling bandit problem arising from a real-life application domain and a large-scale synthetic problem used to test the scalability of SCB. The results of these experiments show that CCB beats all existing Copeland dueling bandit algorithms, while SCB outperforms CCB on the large-scale problem.

One open question raised by our work is how to devise an algorithm that has the benefits of both CCB and SCB, i.e., the scalability of the latter together with the former's better dependence on the gaps. At this point, it is not clear to us how this could be achieved.

Another interesting direction for future work is an extension of both CCB and SCB to problems with a continuous set of arms. Given the prevalence of cyclical preference relationships in practice, we hypothesize that the non-existence of a Condorcet winner is an even greater issue when dealing with an infinite number of arms. Given that both our algorithms utilize confidence bounds to make their choices, we anticipate that continuous-armed UCB-style algorithms like those proposed in \cite{kleinberg:2008,Bubeck:2011,Srinivas:2010,Munos:2011,Bull:2011,deFreitas:2012,Valko:2013} can be combined with our ideas to produce a solution to the continuous-armed Copeland bandit problem that does not rely on the convexity assumptions made by algorithms such as the one proposed in \cite{Yue:2009}.

Finally, it is also interesting to expand our results to handle scores other than the Copeland score, such as an $\eps$-insensitive variant of the Copeland score (as in \cite{Busa-Fekete:2014}), or completely different notions of winners, such as the Borda, the Random Walk or the von Neumann winners (see, e.g., \cite{altman2008axiomatic,CDB:2015}).




\bibliographystyle{unsrt}
\setlength{\bibsep}{1.2pt}
\bibliography{copeland}



\appendix
\section*{}
{\Huge \underline{Appendix}}

%
%
%


\section{Experimental Results}
\label{sec:experiments}

To evaluate our methods CCB and SCB, we apply them to three Copeland dueling bandit problems. The first is a $5$-armed problem arising from \emph{ranker evaluation} in the field of \emph{information retrieval} (IR)~\citep{mann:intr08}.  The second is a $500$-armed synthetic example created to test the scalability of SCB.  The third is an example with a Condorcet winner which shows how CCB compares against RUCB when the condition required by RUCB is satisfied.

All three experiments follow the experimental approach in \citep{YueJoachims:2011,RUCB2014} and use the given preference matrix to simulate comparisons between each pair of arms $(a_i,a_j)$ by drawing samples from Bernoulli random variables with mean $p_{ij}$. We compare our two proposed algorithms against the state of the art $K$-armed dueling bandit algorithm, RUCB \cite{RUCB2014}, and Copeland SAVAGE, PBR and RankEl. We include RUCB in order to verify our claim that $K$-armed dueling bandit algorithms that assume the existence of a Condorcet winner have linear regret if applied to a Copeland dueling bandit problem without a Condorcet winner. Note that in all our plots, the horizontal time axes use a log scale, while the vertical axes, which measure cumulative regret, use a linear scale. 

The first experiment uses a $5$-armed problem arising from \emph{ranker evaluation} in the field of \emph{information retrieval} (IR)~\citep{mann:intr08}, detailed in Appendix \ref{sec:exp-det}. Figure \ref{fig:SmallExp} shows the regret accumulated by CCB, SCB, the Copeland variants of SAVAGE, PBR and RankEl, as well as RUCB on this problem. CCB outperforms all other algorithms in this $5$-armed experiment. 

\begin{figure}[!t]

\centering
\includegraphics[width=.78\columnwidth,clip=true,trim=0mm 2mm 0mm 0mm]{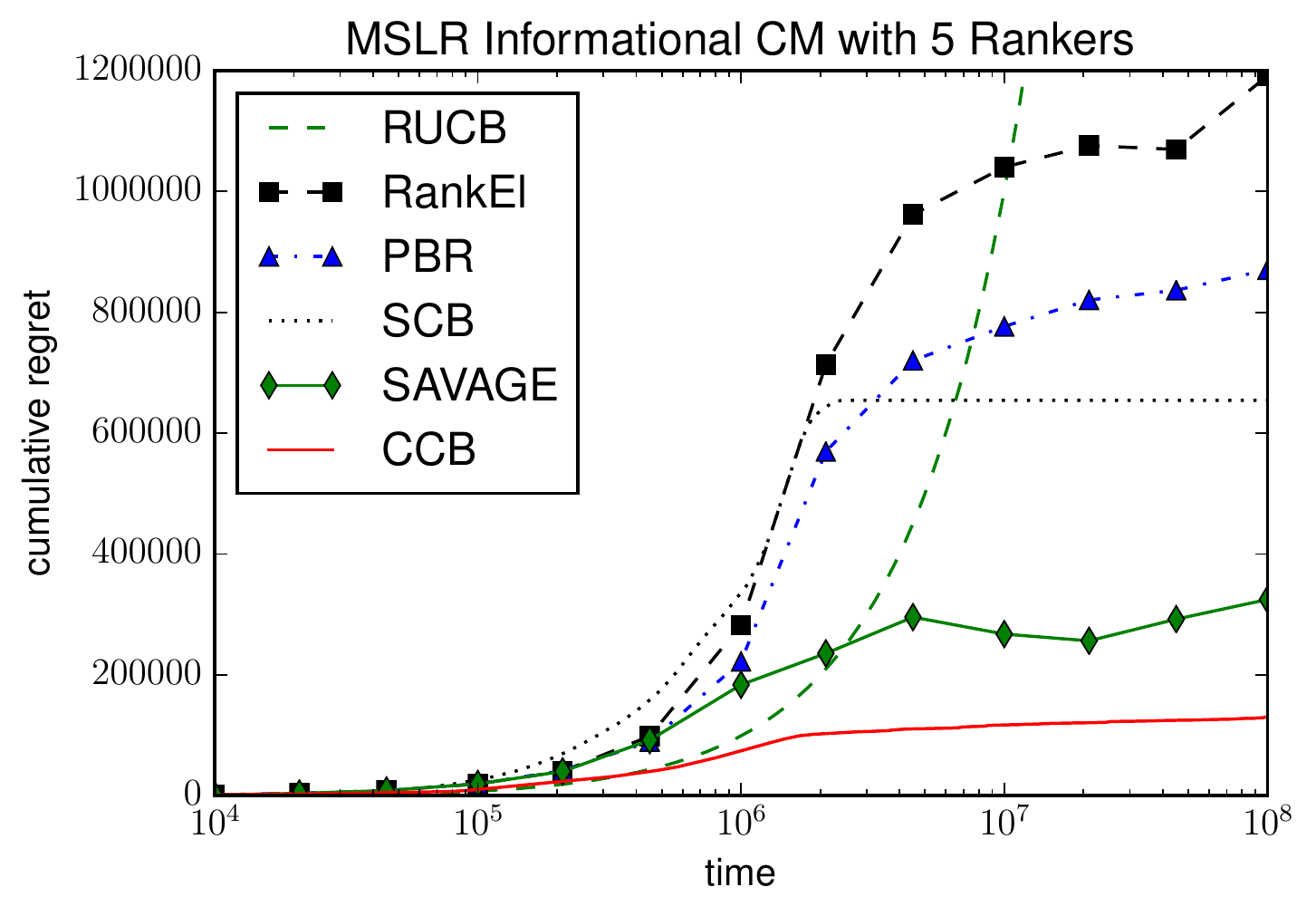}

\vspace{-3mm}

\caption{Small-scale regret results for a 5-armed Copeland dueling bandit problem arising from ranker evaluation.}

\vspace{5mm}

\label{fig:SmallExp}
\end{figure}

Note that three of the baseline algorithms under consideration here (i.e., SAVAGE, PBR and RankEl) require the horizon of the experiment as an input. Therefore, we ran independent experiments with varying horizons and recorded the accumulated regret: the markers on the curves corresponding to these algorithms represent these numbers. Consequently,  the regret curves are not monotonically increasing. For instance, SAVAGE's cumulative regret at time $2\times10^7$ is lower than at time $10^7$ because the runs that produced the former number were not continuations of those that resulted in the latter, but rather completely independent. Furthermore, RUCB's cumulative regret grows linearly, which is why the plot does not contain the entire curve.

\begin{figure}[!b]

\vspace{5mm}

\centering
\includegraphics[width=.65\columnwidth,clip=true,trim=0mm 2mm 0mm 2mm]{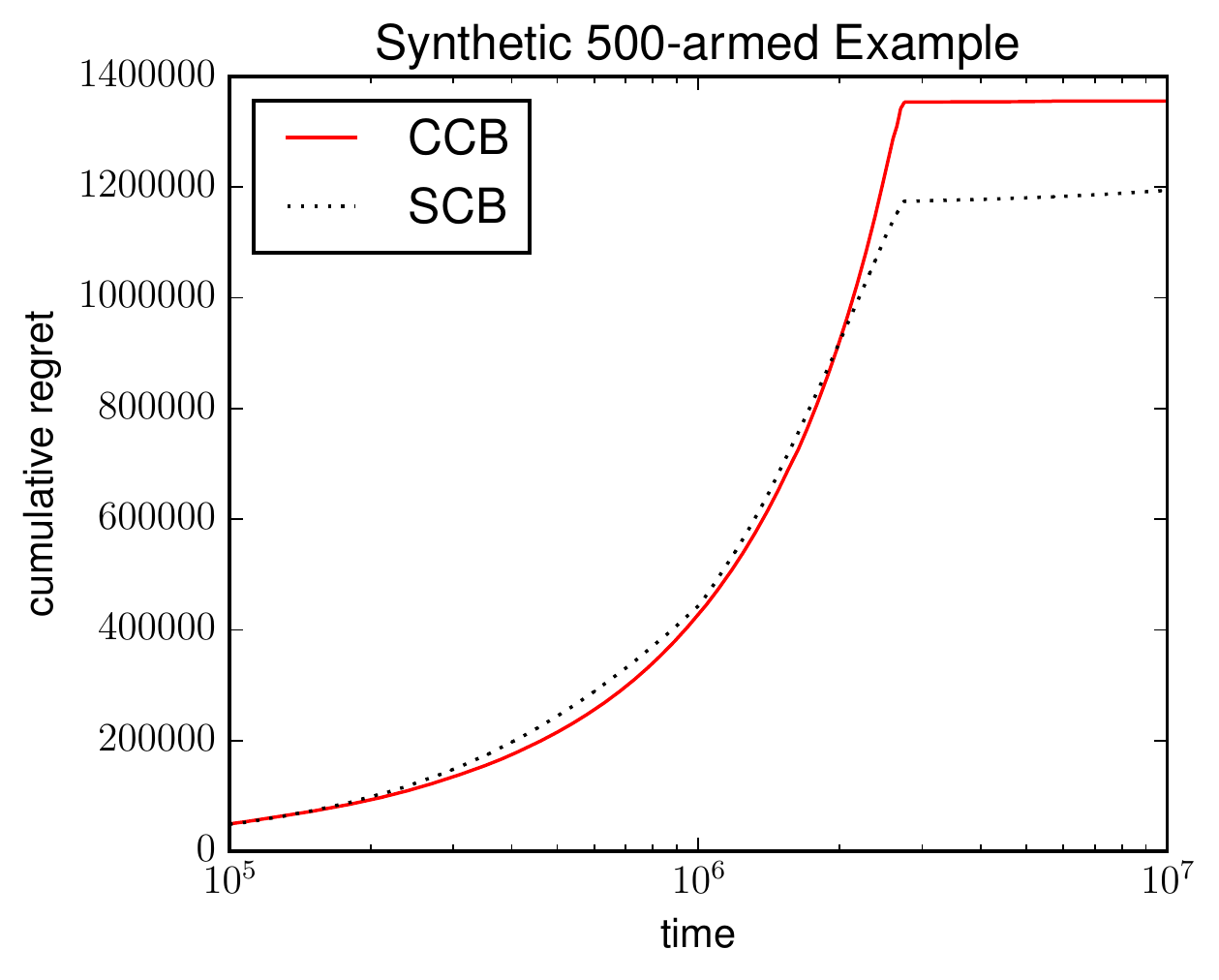}

\vspace{-3mm}

\caption{Large-scale regret results for a synthetic 500-armed Copeland dueling bandit problem.}


\label{fig:LargeExp}
\end{figure}

The second experiment uses a $500$-armed synthetic example created to test the scalability of SCB.
In particular, we fix a preference matrix in which the three Copeland winners are in a cycle, each with a Copeland score of 498, and the other arms have Copeland scores ranging from 0 to 496.

Figure \ref{fig:LargeExp}, which depicts the results of this experiment, shows that when there are many arms, SCB can substantially outperform CCB. We omit SAVAGE, PBR and RankEl from this experiment because they scale poorly in the number of arms \cite{Urvoy:2013,Busa-Fekete:2013,Busa-Fekete:2014}. 

The reason for the sharp transition in the regret curves of CCB and SCB in the synthetic experiment is as follows. Because there are many arms, as long as one of the two arms being compared is not a Copeland winner, the comparison can result in substantial regret; since both algorithms choose the second arm in each round based on some criterion other than the Copeland score, even if the first chosen arm in a given time-step is a Copeland winner, the incurred regret may be as high as $0.5$. The sudden transition in Figure \ref{fig:LargeExp} occurs when the algorithm becomes confident enough of its choice for the first arm to begin comparing it against itself, at which point it stops accumulating regret.


The third experiment is an example with a Condorcet winner designed to show how CCB compares against RUCB when the condition required by RUCB is satisfied. The regret plots for SAVAGE and SCB were excluded here since they both perform substantially worse than either RUCB or CCB, as expected. This example was extracted in the same fashion as the example used in the ranker evaluation experiment detailed in Appendix \ref{sec:exp-det}, with the sole difference that this time we ensured that one of the rankers is a Condorcet winner. The  results, depicted in Figure \ref{fig:Exp2}, show that CCB enjoys a slight advantage over RUCB in this case. We attribute this to the careful process of identifying and utilizing the weaknesses of non-Copeland winners, as carried out by lines 12 and 18 of Algorithm~\ref{alg:CCB}.

\begin{figure}[!h]

\centering
\includegraphics[width=.65\columnwidth,clip=true,trim=0mm 2mm 0mm 0mm]{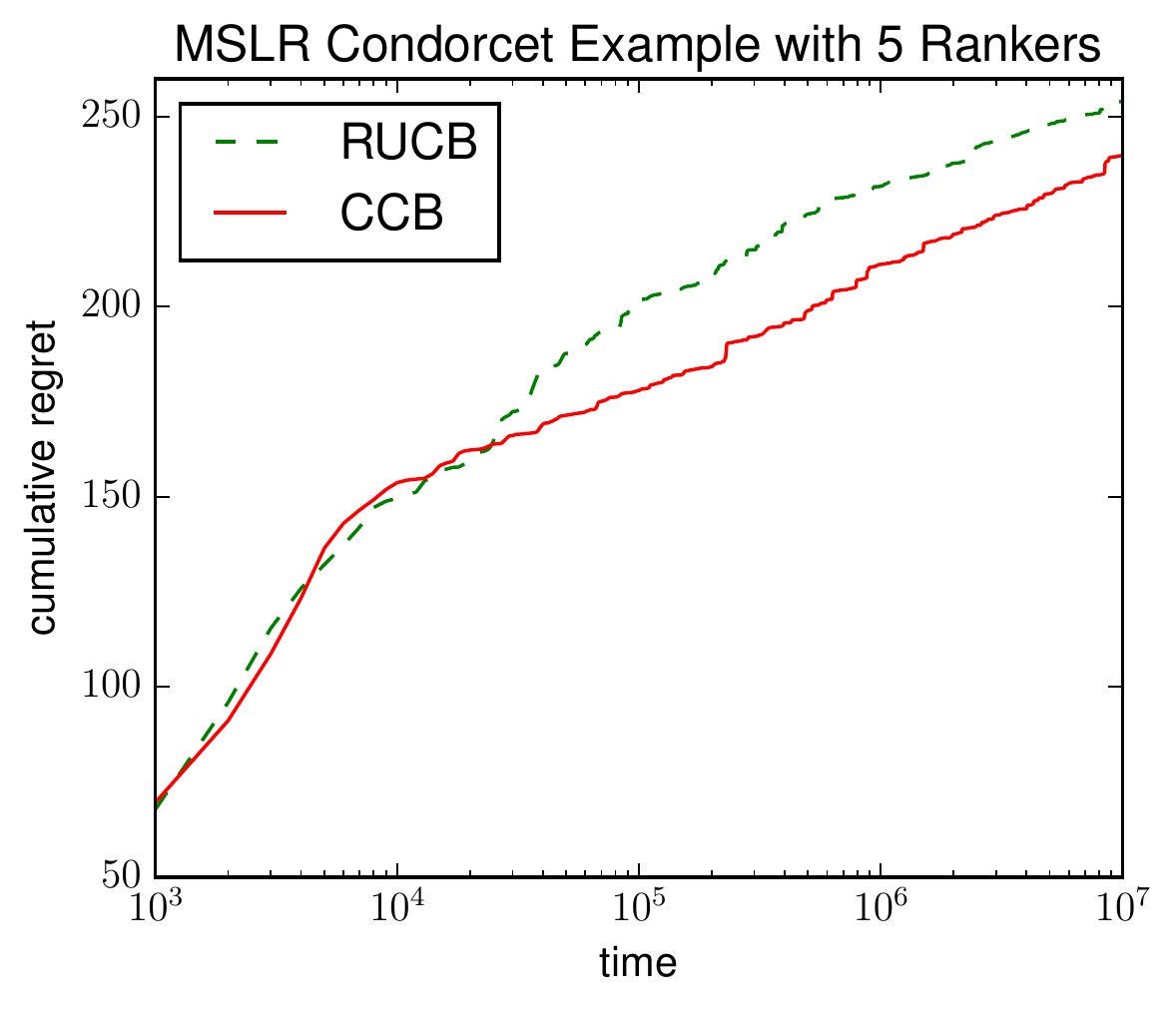}
\vspace*{-1\baselineskip}

\caption{Regret results for a Condorcet example.}

\vspace{5mm}

\label{fig:Exp2}
\end{figure}

\section{Ranker Evaluation Details}
\label{sec:exp-det}

A ranker is a function that takes as input a user's search query and ranks the documents in a collection according to their relevance to that query.  Ranker evaluation aims to determine which among a set of rankers performs best.  One effective way to achieve this is to use \emph{interleaved comparisons}~\citep{radlinski2008:how}, which interleave the ranked lists of documents proposed by two rankers and present the resulting list to the user, whose subsequent click feedback is used to infer a noisy preference for one of the rankers.  Given a set of $K$ rankers, the problem of finding the best ranker can then be modeled as a $K$-armed dueling bandit problem, with each arm corresponding to a ranker. 

We use interleaved comparisons to estimate the preference matrix for the full set of rankers included with the MSLR dataset\footnote{\url{http://research.microsoft.com/en-us/projects/mslr/default.aspx}},
from which we select $5$ rankers such that a Condorcet winner does not exist.
The MSLR dataset consists of relevance judgments provided by expert annotators assessing the relevance of a given document to a given query. Using this data set, we create a set of 136 rankers, each corresponding to a ranking feature provided in the data set, e.g., PageRank.  The ranker evaluation task in this context corresponds to determining which single feature constitutes the best ranker~\citep{hofmann:irj13}.

To compare a pair of rankers, we use \emph{probabilistic interleave} (PI)~\citep{hofmann11:probabilistic}, a recently developed method for interleaved comparisons. To model the user's click behavior on the resulting interleaved lists, we employ a probabilistic user model~\citep{hofmann11:probabilistic,craswell08:experimental} that uses as input the manual labels (classifying documents as relevant or not for given queries) provided with the MSLR dataset. Queries are sampled randomly and clicks are generated probabilistically by conditioning on these assessments in a way that resembles the behavior of an actual user \citep{guo09:efficient}. Specifically, we employ an informational click model in our ranker evaluation experiments~\cite{hofmann13:balancing}. 

The informational click model simulates the behavior of users whose goal is to acquire knowledge about multiple facets of a topic, rather than seeking a specific page that contains all the information that they need. As such, in the informational click model, the user tends to continue examining documents even after encountering a highly relevant document. The informational click model is one of the three click models utilized in the ranker evaluation literature, along with the perfect and navigational click models \cite{hofmann13:balancing}. It turns out that the full preference matrix of the feature vectors of the MSLR dataset has a Condorcet winner when the perfect or the navigational click-models are used. As we will see in Appendix~\ref{sec:condorcet}, using the informational click model that is no longer true. 

Following \citep{YueJoachims:2011,RUCB2014}, we first use the above approach to estimate the comparison probabilities $p_{ij}$ for each pair of rankers and then use these probabilities to simulate comparisons between rankers. More specifically, we estimate the full preference matrix, called the \emph{informational preference matrix}, by performing $400,000$ interleaved comparisons on each pair of the $136$ feature rankers.



\section{Assumptions and Key Quantities}
\label{sec:assumptions}

\begin{figure}[!b]


\centering
\includegraphics[width=.75\textwidth]{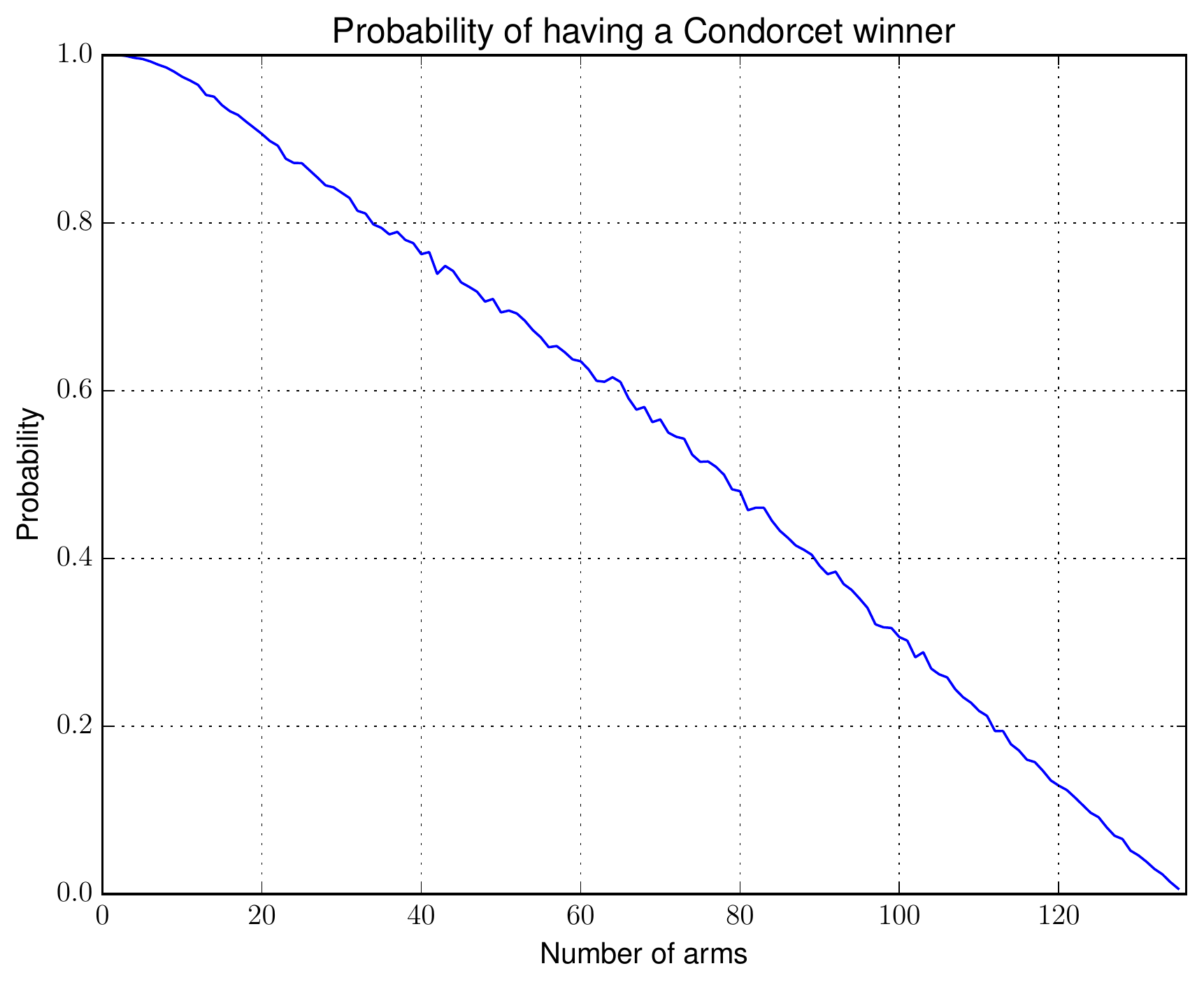}

\vspace{-3mm}

\caption{The probability that the Condorcet assumption holds for subsets of the feature rankers. The probability is shown as a function of the size of the subset.}

\vspace{3mm}

\label{fig:prob-condorcet}
\end{figure}

In this section, we provide quantitative analysis of the various assumptions, definitions and quantities that were discussed in the main body of the paper.

\subsection{The Condorcet Assumption}
\label{sec:condorcet}

To test how stringent the Condorcet assumption is, we use the informational preference matrix described in Section \ref{sec:exp-det} to estimate for each $K=1,\ldots,136$ the probability $P_K$ that a given $K$-armed dueling bandit problem, obtained from considering $K$ of our $136$ feature rankers, would have a Condorcet winner by randomly selecting $10,000$ $K$-armed dueling bandit problems and counting the ones with Condorcet winners. As can be seen from Figure \ref{fig:prob-condorcet}, as $K$ grows the probability that the Condorcet assumption holds decreases rapidly. \masrour{We hypothesize that this is because the informational click model explores more of the list of ranked documents than the navigational click model, which was used in \cite{RUCB2014}, and so it is more likely to encounter non-transitivity phenomena of the sort described in \cite{gardner:1970}.}


%

\subsection{Other Notions of Winners}
\label{sec:other-winners}

As mentioned in Section \ref{sec:related-work}, numerous other definitions of what constitutes the best arm have been proposed, some of which specialize to the Condorcet winner, when it exists. This latter property is desirable both in preference learning and social choice theory: the Condorcet winner is the choice that is preferred over all other choices, so if it exists, there is good reason to insist on selecting it. The Copeland winner, as discussed in this paper, and the von Neumann winner \cite{CDB:2015} satisfy this property, while the Borda (a.k.a.\ Sum of Expectations) and the Random Walk (a.k.a.\ PageRank) winners \cite{busa2014survey} do not.
The von Neumann winner is in fact defined as a distribution over arms such that playing it will maximize the probability to beat any fixed arm. The Borda winner is defined as the arm maximizing the score $\sum_{j \neq i} p_{ij}$ and can be interpreted as the arm that beats other arms by the most, rather than beating the most arms. The Random Walk winner is defined as the arm we are most likely to visit in some Markov Chain determined by the preference matrix.
In this section, we provide some numerical evidence for the similarity of these notions in practice, based on the sampled preference matrices obtained from the ranker evaluation from IR, which was described in the last section. Table \ref{tbl:winners} lists the percentage of preference matrices for which pairs of winner overlapped. In the case of the von Neumann winner, which is defined as a probability distribution over the set of arms \cite{CDB:2015}, we used the support of the distribution (i.e., the set of arms with non-zero probability) to define overlap with the other definitions.

\begin{table}[h]

\vspace{-5mm}

\centering
    \caption{Percentage of matrices for which the different notions of winners overlapped}
    \label{tbl:winners}
    \vspace{2mm}
    \begin{tabular}{c|cccc}
	 Overlap & Copeland & von Neumann & Borda & Random Walk  \\
	 \hline
	 Copeland & \%100   &  \%99.94 &  \%51.49 &  \%56.15 \\
	 von Neumann & \%99.94 &  \%100   &  \%77.66 &  \%82.11 \\
	 Borda & \%51.49 &  \%77.66 &  \%100   &  \%94.81 \\
	 RandomWalk & \%56.15 &  \%82.11 &  \%94.81 &  \%100   
     \end{tabular}


\end{table}

As these numbers demonstrate, the Copeland and the von Neumann winners are very likely to overlap, as are the Borda and Random Walk winners, while the first two definitions are more likely to be incompatible with the latter two. 
Furthermore, in the case of \%94.2 of the preference matrices, all Copeland winners were contained in the support of the von Neumann winner, suggesting that in practice the Copeland winner is a more restrictive notion of what constitutes a winner.  

\begin{figure}[!b]

\vspace{-2mm}

\centering
\includegraphics[width=.97\textwidth]{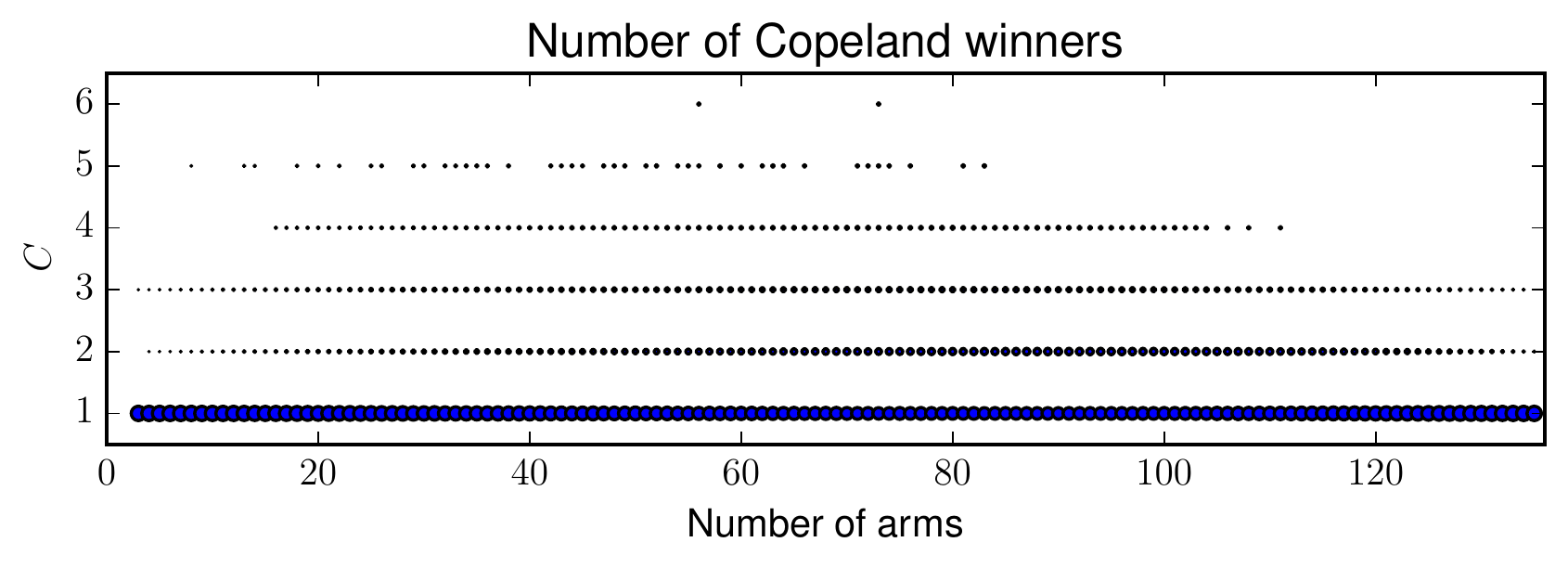}
\includegraphics[width=.97\textwidth]{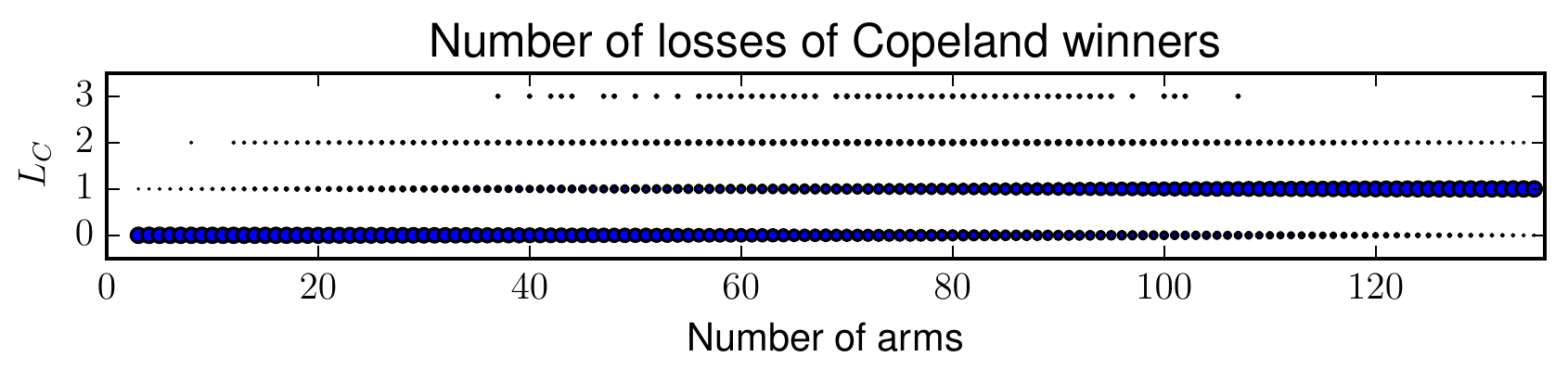}

\vspace{-3mm}

\caption{Observed values of the parameters $C$ and $L_C$: the area of the circle with coordinates $(x,y)$ is proportional to the percentage of examples with $K=x$ which satisfied $C=y$ (in the top plot) or $L_C=y$ in the bottom plot.}


\label{fig:CandL_C}
\end{figure}

\subsection{The Quantities $C$ and $L_C$}
\label{sec:C-L_C}

We also examine additional quantities relevant to our regret bounds: the number of Copeland winners, $C$; the number of losses of each Copeland winner, $L_C$; and the range of values in which these quantities fall. Using the above randomly chosen preference sub-matrices, we counted the number of times each possible value for $C$ and $L_C$ was observed. The results are depicted in Figure \ref{fig:CandL_C}: the area of the circle with coordinates $(x,y)$ is proportional to the percentage of examples with $K=x$ which satisfied $C=y$ (in the top plot) or $L_C=y$ (in the bottom plot). As these plots show, the parameters $C$ and $L_C$ are generally much lower than $K$.

\subsection{The Gap $\Delta$}
\label{sec:Delta}

The regret bound for CCB, given in \eqref{eqn:ExpRegret}, depends on the gap $\Delta$ defined in Definition \ref{def:deltas}.\ref{item:Delta}, rather than the smallest gap $\Delta_{\min}$ as specified in Definition \ref{def:deltas}.\ref{item:Delta_min}. The latter would result in a looser regret bound and Figure \ref{fig:gap-ratios} quantifies this deterioration in the ranker evaluation example under consideration here. In particular, the plot depicts the average of the ratio between the two bounds (the one using $\Delta$ and the one using $\Delta_{\min}$) across the $10,000$ sampled preference matrices used in the analysis of the Condorcet winner for each $K$ in the set $\{2,\ldots,135\}$. The average ratio decreases as the number of arms approaches $136$ because, as $K$ increases, the sampled preference matrices increasingly resemble the full preference matrix and so their gaps $\Delta$ and $\Delta_{\min}$ approach those of the full $136$-armed preference matrix as well. As it turns out, the ratio $\Delta^2/\Delta_{\min}^2$ for the full matrix is equal to $1,419$. Hence, the curve in Figure \ref{fig:gap-ratios} approaches that number as the number of arms approaches $136$.

\begin{figure}[!h]
\centering
\includegraphics[width=.85\textwidth]{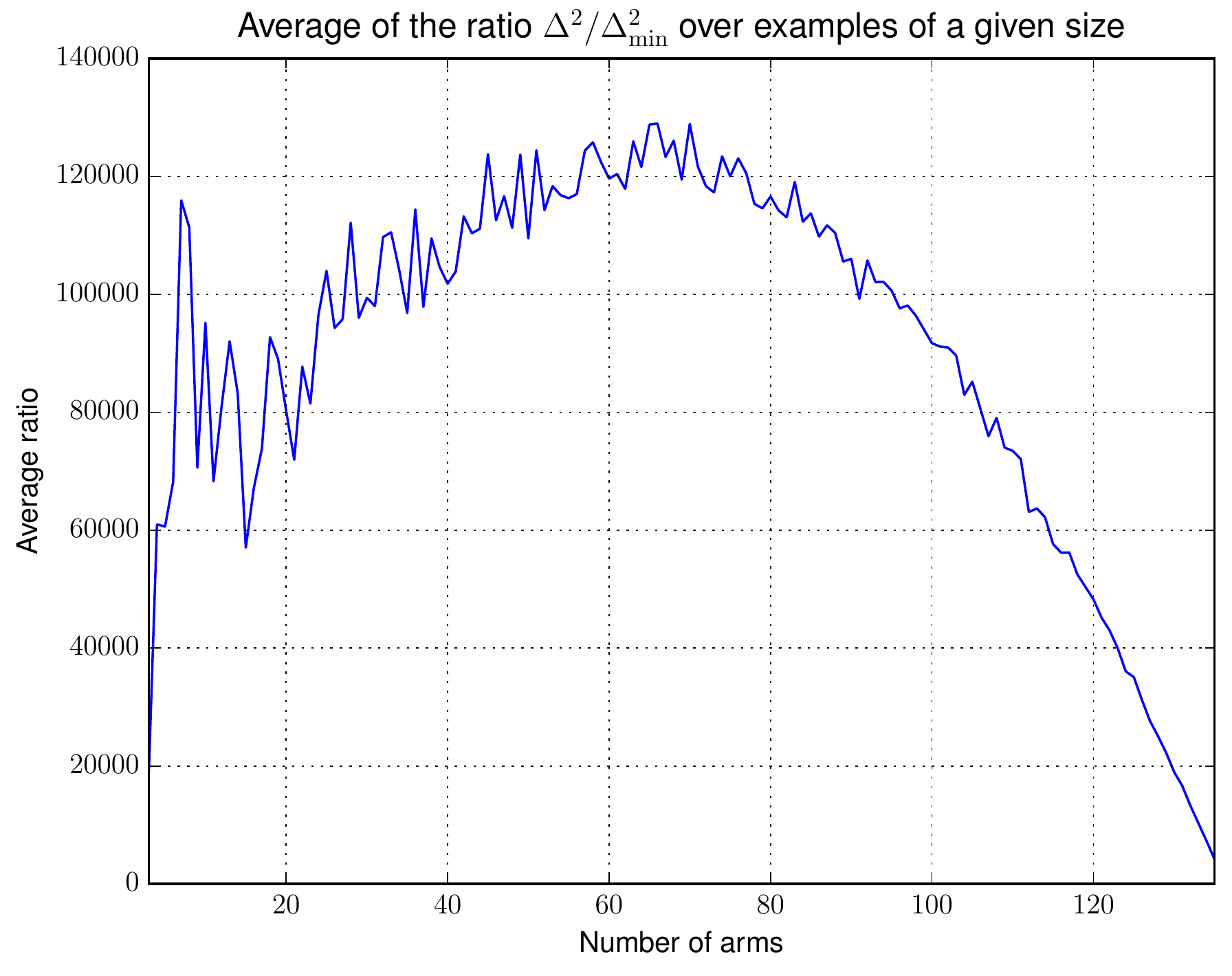}

\vspace{-3mm}

\caption{The average advantage gained by having the bound in~\eqref{eqn:ExpRegret} depend on $\Delta$ rather than $\Delta_{\min}$: for each number of arms $K$, the expectation is taken across the $10,000$ $K$-armed preference matrices obtained using the sampling procedure described above.}


\label{fig:gap-ratios}
\end{figure}

\newpage

\section{Background Material}
\label{sec:prelim}

\noindent{\bf Maximal Azuma-Hoeffding Bound} \citep[\S A.1.3]{Cesa-Bianchi:2006}: Given random variables $X_1,\ldots,X_N$ with common range $[0,1]$ satisfying $\mathbf{E}[X_n | X_1, \ldots, X_{n-1} ] = \mu$, define the partial sums $S_n = X_1 + \cdots + X_n$. Then, for all $a > 0$, we have
\begin{align*}
P\Big(\max_{n\leq N} S_n > n\mu & + a\Big) \leq e^{-2a^2/N} \\
P\Big(\min_{n\leq N} S_n < n\mu & - a\Big) \leq e^{-2a^2/N}
\end{align*}

Here, we will quote a useful Lemma that we will refer to repeatedly in our proofs:

\begin{lemma}[Lemma 1 in \cite{RUCB2014}]\label{lem:HighProbBound} Let $\vP := \left[ p_{ij} \right]$ be the preference matrix of a $K$-armed dueling bandit problem with arms $\{a_1,\ldots,a_K\}$. Then, for any dueling bandit algorithm and any $\alpha > \frac{1}{2}$ and $\delta > 0$, we have
\begin{equation*}  P\Big( \forall\,t>C(\delta),i,j,\; p_{ij} \in [l_{ij}(t),u_{ij}(t)] \Big) > 1-\delta. \end{equation*}
\end{lemma}

\section{Proof of Proposition \ref{prop:counts}}
\label{sec:counts-proof}

Before starting with the proof, let us point out the following two properties that can be derived from Assumption {\bf A} in Section \ref{sec:theory}:

\begin{itemize}[topsep=0mm,parsep=0pt,itemsep=0pt,partopsep=0pt]
\item[{\bf P1}]\label{P1} There are no ties involving a Copeland winner and a non-Copeland winner, i.e., for all pairs of arms $(a_i,a_j)$ with $i \leq C < j$, we have $p_{ij} \neq 0.5$.
\item[{\bf P2}]\label{P2} Each non-Copeland winner has more losses than every Copeland winner, i.e., for every pair of arms $(a_i,a_j)$, with $i \leq C < j$, we have $|\mathcal{L}_i| < |\mathcal{L}_j|$.
\end{itemize}


\masrour{Even though we have assumed in the statement of Proposition \ref{prop:counts} that Assumption {\bf A} holds, it turns out that the proof provided in this section holds as long as the above two properties hold.}

{\bf Proposition \ref{prop:counts}} \emph{Applying CCB to a dueling bandit problem satisfying properties {\bf P1} and {\bf P2}, we have the following bounds on the number of comparisons involving various arms for each $T>C(\delta)$: for each pair of arms $a_i$ and $a_j$, such that either at least one of them is not a Copeland winner or $p_{ij} \neq 0.5$, with probability $1-\delta$ we have}
\begin{equation}\label{eqn:counts}
N_{ij}^{\delta}(T) \leq \widehat{N}_{ij}^{\delta}(T) := \left\{\begin{array}{cl} \dfrac{4\alpha\ln T}{\left(\Delta^*_{ij}\right)^2} & \textup{if } i \neq j \\[15pt] 0 & \textup{if $i=j>C$} \end{array}\right.
\end{equation}


\begin{proof}[Proof of Proposition \ref{prop:counts}] We will prove these bounds by considering a number of cases separately:
\begin{enumerate}[leftmargin=*]
\item\label{item:CNC} {\bf $i \leq C$ and $p_{ij} \neq 0.5$}: 
First of all, since $a_i$ is a Copeland winner, this means that according to the definitions in Tables \ref{tbl:notation1} and \ref{tbl:notation2}, $\Delta_{ij}^*$ is simply equal to $\Delta_{ij}$; secondly, assuming by way of contradiction that $N_{ij}^\delta(t) > \frac{4\alpha\ln T}{\Delta_{ij}} > 0$, 
then we have $\tau_{ij} > C(\delta)$ and so by Lemma \ref{lem:HighProbBound}, we have with probability $1-\delta$ that the confidence interval $[l_{ij}(\tau_{ij}),u_{ij}(\tau_{ij})]$ contains the preference probability $p_{ij}$. But, in order for arm $a_j$ to have been chosen as the challenger to $a_i$, we must also have $0.5 \in [l_{ij}(\tau_{ij}),u_{ij}(\tau_{ij})]$; to see this, let us consider the two possible cases:
\begin{enumerate}[leftmargin=*]
\item If we have $p_{ij} > 0.5$, then having
\[ 0.5 \notin [l_{ij}(\tau_{ij}),u_{ij}(\tau_{ij})] \]
implies that we have $l_{ij}(\tau_{ij}) > 0.5$, which in turn implies 
\[ u_{ji}(\tau_{ij}) = 1-l_{ij}(\tau_{ij}) < 0.5 = u_{ii}(\tau_{ij}), \]
but this is impossible since in that case $a_i$ would've been chosen as the challenger.
\item If we have $p_{ij} < 0.5$, then have
\[ 0.5 \notin [l_{ij}(\tau_{ij}),u_{ij}(\tau_{ij})] \]
implies that we have $u_{ij}(\tau_{ij}) < 0.5$, but this is impossible because it means that we had $l_{ji}(\tau_{ij}) > 0.5$, and CCB would've eliminated it from considerations in its second round.
\end{enumerate}
So, in either case, we cannot have $0.5 \notin [l_{ij}(\tau_{ij}),u_{ij}(\tau_{ij})]$. Therefore, at time $\tau_{ij}$, we must have had $u_{ij}(\tau_{ij})-l_{ij}(\tau_{ij}) > |p_{ij}-0.5| =: \Delta_{ij}$. From this, we can conclude the following, using the definition of $u_{ij}$ and $l_{ij}$:
\begin{align*}
 & u_{ij}(\tau_{ij})-l_{ij}(\tau_{ij}) := 2\sqrt{\frac{\alpha \ln \tau_{ij}}{N_{ij}(\tau_{ij})}} \geq \Delta_{ij} \\
 \therefore\quad & 2\sqrt{\frac{\alpha \ln \tau_{ij}}{N_{ij}^\delta(\tau_{ij})}} \geq \Delta_{ij} \quad \because \; N_{ij}^\delta(\tau_{ij}) \leq N_{ij}(\tau_{ij}) \\
 \therefore\quad & 2\sqrt{\frac{\alpha \ln T}{N_{ij}^\delta(\tau_{ij})}} \geq \Delta_{ij} \quad \because \; \tau_{ij} \leq T \\
 \therefore\quad & N_{ij}^\delta(\tau_{ij}) \leq \frac{4\alpha \ln T}{\Delta_{ij}^2},
\end{align*}

giving us the desired bound. The reader is referred to Figure \ref{fig:CopelandNonCopelandFigure} for an illustration of this argument.  

\begin{figure}[t]

\centering
\includegraphics[width=.95\textwidth]{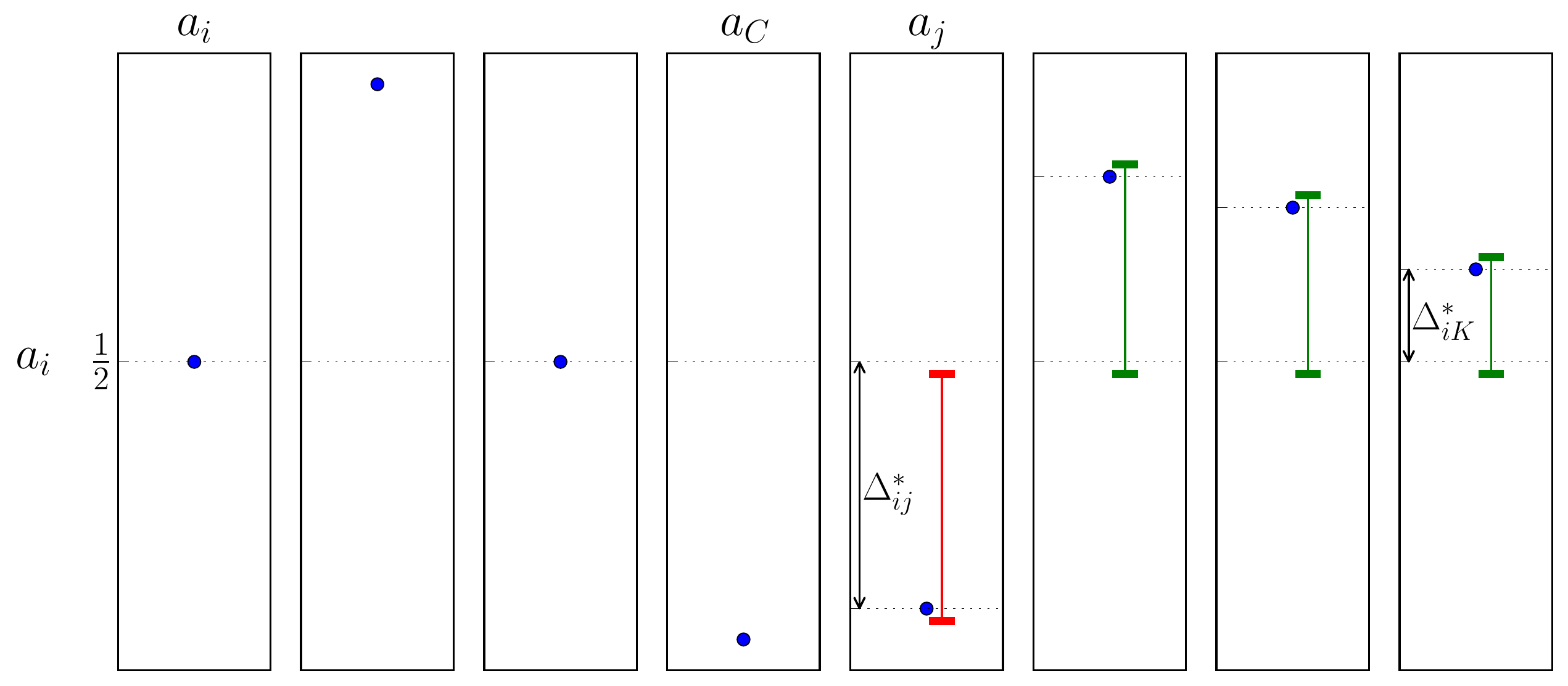}


\caption{\masrour{This figure illustrates the definition of the quantities $\Delta^*_i$ and $\Delta^*_{ij}$ in the case that arm $a_i$ is a Copeland winner, as well as the idea behind Case 1 in the proof of Proposition \ref{prop:counts}. In this setting we have $\Delta^*_i = 0$ and $\Delta^*_{ij} = \Delta_{ij}$. On the one hand, by Lemma \ref{lem:HighProbBound}, we know that the confidence intervals will contain the $p_{ij}$ (the blue dots in the plots), and on the other as soon as the confidence interval of $p_{ij}$ stops containing $0.5$ for some arm $a_j$, we know that it could not be chosen to be compared against $a_i$. In this way, the gaps $\Delta^*_{ij}$ regulate the number of times that arm each arm can be chosen to be played against $a_i$ during time-steps when $a_i$ is chosen as optimistic Copeland winner.}} 
\label{fig:CopelandNonCopelandFigure}

\vspace{5mm}

\end{figure}

\item {\bf $C < i$}: Let us deal with the two cases included in Inequality \eqref{eqn:counts} separately:
\begin{enumerate}[leftmargin=*]
\item {$i=j>C$:} In plain terms, this says that with probability $1-\delta$ no non-Copeland winner will be compared against itself after time $C(\delta)$. The reason for this is the following set of facts:
\begin{itemize}[leftmargin=*]
\item Since $a_i$ is a non-Copeland winner, we have by Property {\bf P1} that it loses to more arms than any Copeland winner.
\item For $a_i$ to have been chosen as an optimistic Copeland winner, it has to have (optimistically) lost to no more than $L_C$ arms, which means that there exists an arm $k$ such that $p_{ik} < 0.5$, but $u_{ik} \geq 0.5$.
\item By Lemma \ref{lem:HighProbBound}, for all time steps after $C(\delta)$, we have $l_{ik} \leq p_{ik} < 0.5$, and so in the second round we have $u_{ki} > 0.5 = u_{ii}$, and so $a_i$ could be not chosen as the challenger to itself. 
\end{itemize}
\item {$i \neq j$:} In the case that $a_i$ is not a Copeland winner and $a_j$ is different from $a_i$, we distinguish between the following two cases, where $\Delta_i^*$ is defined as in Tables \ref{tbl:notation1} and \ref{tbl:notation2}:
\begin{enumerate}[leftmargin=4mm]
\item $p_{ij} \leq 0.5-\Delta_i^*$: In this case, the definition of $\Delta_i^*$ reduces to $\Delta_{ij}$. 
Now, since when choosing the challenger, CCB 
eliminates from consideration any arm $a_j$ that has $l_{ji} > 0.5$, the last time-step $\tau_{ij}$ after $C(\delta)$ when $a_j$ was chosen as the challenger for $a_i$, we must've had $u_{ij}(\tau_{ij}) := 1-l_{ji}(\tau_{ij}) \geq 0.5$. On the other hand, Lemma \ref{lem:HighProbBound} implies that we must also have $l_{ij}(\tau_{ij}) \leq p_{ij}$, and therefore, we have $u_{ij}(\tau_{ij})-l_{ij}(\tau_{ij}) \geq \Delta_{ij}$; so, doing the same calculation as in part \ref{item:CNC} of this proof, we have
\begin{align*}
\hspace{-10mm} & u_{ij}(\tau_{ij})-l_{ij}(\tau_{ij}) := 2\sqrt{\frac{\alpha \ln \tau_{ij}}{N_{ij}(\tau_{ij})}} \geq \Delta_{ij} \\
\hspace{-10mm} \therefore\quad & 2\sqrt{\frac{\alpha \ln \tau_{ij}}{N_{ij}^\delta(\tau_{ij})}} \geq \Delta_{ij} \quad \because \; N_{ij}^\delta(\tau_{ij}) \leq N_{ij}(\tau_{ij}) \\
\hspace{-10mm} \therefore\quad & 2\sqrt{\frac{\alpha \ln T}{N_{ij}^\delta(\tau_{ij})}} \geq \Delta_{ij} \quad \because \; \tau_{ij} \leq T \\
\hspace{-10mm} \therefore\quad & N_{ij}^\delta(\tau_{ij}) \leq \frac{4\alpha \ln T}{\Delta_{ij}^2},
\end{align*}

\item $p_{ij} > 0.5-\Delta_i^*$: Repeating the above argument about $u_{ij}(\tau_{ij})$, we can deduce that $u_{ij}(\tau_{ij}) \geq 0.5$ must hold. On the other hand, Lemma \ref{lem:HighProbBound} states that with probability $1-\delta$ we have $u_{ij}(\tau_{ij}) \geq p_{ij}$. Putting these two together we get
\begin{equation}\label{eqn:NCub} u_{ij}(\tau_{ij}) \geq \max\{0.5,p_{ij}\}. \end{equation}
On the other hand, we will show next that with probability $1-\delta$, we have $l_{ij}(\tau_{ij}) \leq 0.5-\Delta_i^*$; this is a consequence of the following facts:
\begin{itemize}[leftmargin=*]
\item Since $a_i$ was chosen as the optimistic Copeland winner, we can deduce that $a_i$ had no more that $L_C$ optimistic losses.
\item Let $a_{k_1},\ldots,a_{k_{l}}$ be the $l \leq L_C$ arms to which $a_i$ lost optimistically during time-step $\tau_{ij}$. Then, the smallest $p_{ik}$ with $k \notin \{k_1,\ldots,k_{l}\}$, must be less than to equal to the $\{L_C+1\}^{th}$ smallest element in the set $\{p_{ik} \, | \, k=1,\ldots,K \}$. 
\item This, in turn, is equal to the $\{L_C+1\}^{th}$ smallest element in the set $\{p_{ik} | p_{ik} < 0.5 \}$ (since this latter set of numbers are the smallest ones in the former set). But, this is equal to $0.5-\Delta_i^*$ by definition.
\end{itemize}
So, we have the desired bound on $l_{ij}(\tau_{ij})$ and combining this with Inequality \eqref{eqn:NCub}, we have
\[ \hspace{-5mm} u_{ij}(\tau_{ij})-l_{ij}(\tau_{ij}) \geq \max\{0,p_{ij}-0.5\}+\Delta_i^* = \Delta_{ij}^*, \]
where the last equality follows directly from the definition of $\Delta_{ij}^*$ and the fact that $p_{ij} > 0.5-\Delta_i^*$. Now, repeating the same calculations as before, we can conclude that with probability $1-\delta$, we have
\[ N_{ij}^\delta(\tau_{ij}) \leq \frac{4\alpha \ln T}{\left(\Delta_{ij}^*\right)^2}. \]
\end{enumerate}
\end{enumerate}

A pictorial depiction of the various steps in this part of the proof can be found in Figure \ref{fig:nonCopelandFigure}. \qedhere
\end{enumerate}

\end{proof}

\newpage

\begin{figure}[t]

\centering
\includegraphics[width=.95\textwidth]{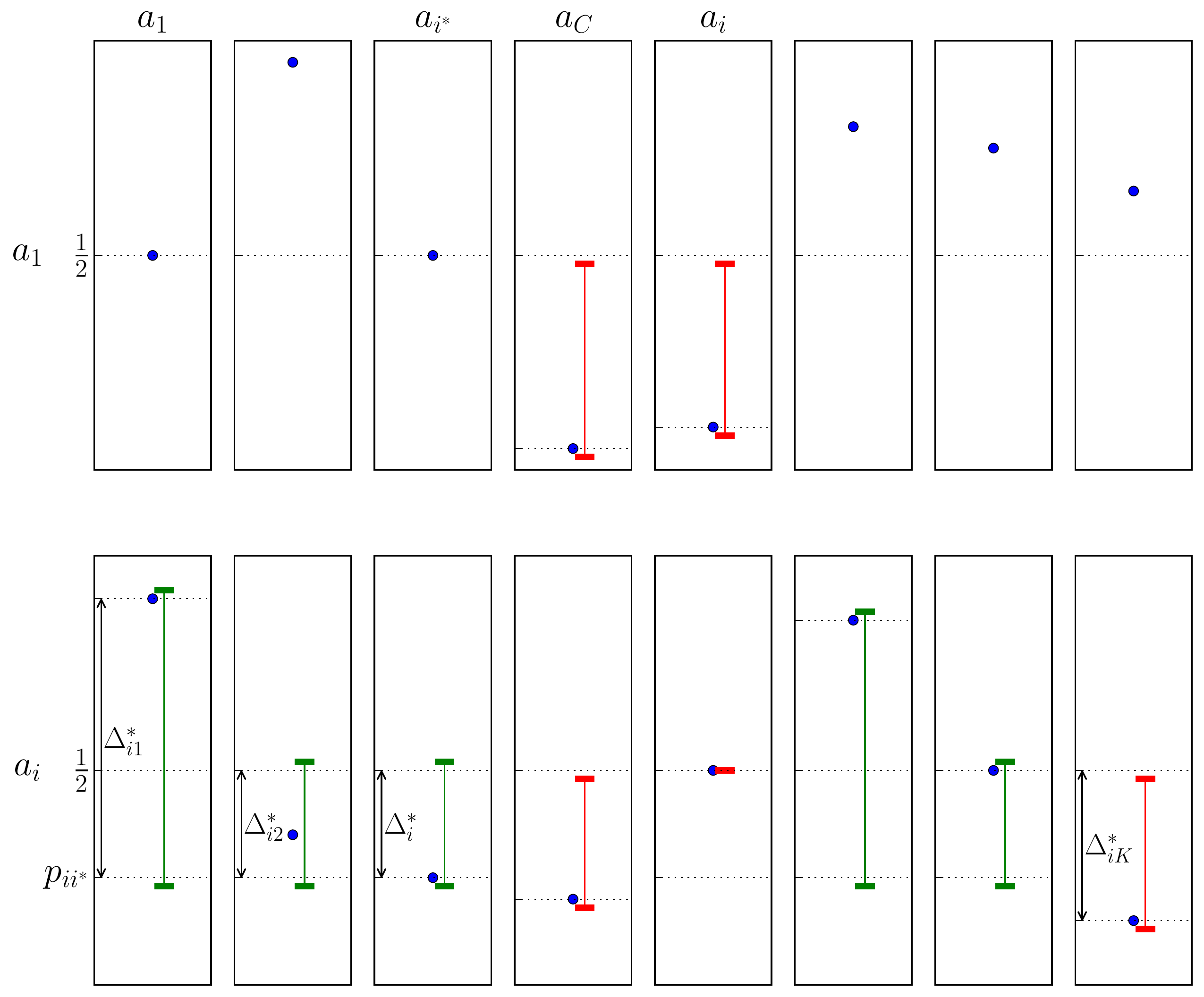}

\vspace{-3mm}

\caption{This figure illustrates the definition of the quantities $\Delta^*_i$ and $\Delta^*_{ij}$ in the case that arm $a_i$ is not a Copeland winner, as well as the idea behind Case 2 in the proof of Proposition \ref{prop:counts}. The bottom row of plots in the figure corresponds to the confidence intervals around probabilities $p_{ij}$ (depicted using the blue dots) for $j=1,\ldots,K$, while the top row corresponds to those for probabilities $p_{1j}$, where $a_1$ is by assumption one of the Copeland winners (although we could use any other Copeland winner instead). \vspace{1mm} \protect\\ %
The two boxes in the top row with red intervals represent arms to which $a_1$ loses (i.e. $p_{1j} < 0.5$), the number of which happens to be $2$ in this example, which means that $L_C = 2$. Now, by Definition \ref{def:deltas}.\ref{item:i*}, $i^*$ is the index with the index $j$ with the $(L_C+1)^{th}$ (in this case $3^{rd}$) lowest $p_{ij}$, and since the three lowest $p_{ij}$ in this example are $p_{iK}, p_{iC}$ and $p_{ii^*}$, this means that the column labeled as $a_{i^*}$ is indeed labeled correctly. Given this, Definition \ref{def:deltas}.\ref{item:Delta*i} tells us that $\Delta^*_i$ is the size of the gap shown in the block corresponding to pair $(a_i,a_{i^*})$. \vspace{1mm} \protect\\ %
Moreover, by Definition \ref{def:deltas}.\ref{item:Delta*ij}, the gap $\Delta^*_{ij}$ is defined using one of the following three cases: (1) if we have $p_{ij} < p_{ii^*}$ (as with the ones with red confidence intervals in the bottom row of plots), then we get $\Delta^*_{ij} := \Delta_{ij} = 0.5-p_{ij}$; (2) if we have $p_{ii^*} < p_{ij} \leq 0.5$ (as in the plots in the $2^{nd}, 3^{rd}$ and $7^{th}$ column of the bottom row), then we get $\Delta^*_{ij} := \Delta^*_i$; (3) if we have $0.5 < p_{ij}$ (as in the $1^{st}$ and $6^{th}$ column in the bottom row), then we get $\Delta^*_{ij} := \Delta_{ij} + \Delta^*_i$. \vspace{1mm} \protect\\ %
The reasoning behind this trichotomy is as follows: in the case of arms $a_j$ in group (1), they are not going to be chosen to be played against $a_i$ as soon as top of the interval goes below $0.5$, and by Lemma \ref{lem:HighProbBound}, we know that the bottom of the interval will be below $p_{ij}$. In the case of the arms in groups (2) and (3), the bottom of their interval needs to be below $p_{ii^*}$ because otherwise that would mean that neither arm $a_{i^*}$ nor arms in group (1) were eligible to be included in the $\argmax$ expression in Line 13 of Algorithm \ref{alg:CCB}, which can only happen if we have $u_{ij} < 0.5$ for $j=i^*$ as well as the arms in group (1), from which we can deduce that the optimistic Copeland score of $a_i$ must have been lower than $K-1-L_C$, and so $a_i$ could not have been chosen as an optimistic Copeland winner. Using the same argument, we can also see that the tops of the confidence intervals corresponding to arms in group (2) must be above $0.5$, or else it would be impossible for $a_i$ to be chosen as an optimistic Copeland winner. Moreover, by Lemma \ref{lem:HighProbBound}, the intervals of the arms $a_j$ in group (3) must contain $p_{ij}$.} 
\label{fig:nonCopelandFigure}

\end{figure}

\clearpage

\section{Proof of Lemma \ref{lem:SetsB}}
\label{sec:SetsB-proof}

Let us begin with the following direct corollary of Proposition \ref{prop:counts}:
\begin{corollary}\label{cor:auto-comp}
Given any $\delta > 0$, any $T > C(\delta)$ and any sub-interval of length $\widehat{N}^{\delta}(T):=\sum_{i \neq j} \widehat{N}_{ij}^{\delta}(T)+1$, with probability $1-\delta$, there is at least one time-step when there exists $c \leq C$ such that 
\begin{align}
\underline{\textup{Cpld}}(a_c) = \textup{Cpld}(a_c) & = \overline{\textup{Cpld}(a_c)} \nonumber \\
& \geq \overline{\textup{Cpld}(a_j)}\;\; \forall\, j, \label{eqn:topCpld}
\end{align}
\end{corollary}

\begin{proof}
According to Proposition \ref{prop:counts}, with probability $1-\delta$, there are at most $\sum_{i \neq j} \widehat{N}^{\delta}_{ij}(T)$ time-steps between $C(\delta)$ and $T$ when Algorithm \ref{alg:CCB}  did not compare a Copeland winner against itself: i.e. $c$ and $d$ in Algorithm \ref{alg:CCB} did not satisfy $c=d \leq C$. 

In other words, during this time-period, in any sub-interval of length $\widehat{N}^{\delta}(T):=\sum_{i \neq j} \widehat{N}_{ij}^{\delta}(T)+1$, there is at least one time-step when a Copeland winner was compared against itself. During this time-step, we must have had 
\begin{align*}
\underline{\textup{Cpld}}(a_c) = \textup{Cpld}(a_c) & = \overline{\textup{Cpld}(a_c)} \\
& \geq \overline{\textup{Cpld}(a_j)}\;\; \forall\, j,
\end{align*}
where the first two equalities are due to the fact that in order for Algorithm \ref{alg:CCB} to set $c=d$, we must have $0.5 \notin [l_{cj},u_{cj}]$ for each $j \neq c$, or else $a_c$ would not be played against itself; on the other hand, the last inequality is due to the fact that $a_c$ was chosen as an optimistic Copeland winner by Line 8 of Algorithm \ref{alg:CCB}, so its optimistic Copeland score must have been greater than or equal to the optimistic Copeland score of the rest of the arms.
\end{proof}

\begin{lemma}\label{lem:preBt}
If there exists an arm $a_i$ with $i>C$ such 
that $\mathcal{B}^i_{C(\delta/2)}$ contains an arm $a_j$ that loses to $a_i$ (i.e. $p_{ij} > 0.5$) or such that $\mathcal{B}^i_{C(\delta/2)}$ contains fewer than $L_C+1$ arms, then the probability that by time-step $T_0$ the sets $\mathcal{B}^i_t$ and $\mathcal{B}_t$ are not reset by Line 9.A of Algorithm \ref{alg:CCB} is less than $\delta/6$, where we define
\begin{align*}
T_0 := C(\delta/2) & + \widehat{N}^{\delta/2}(T_{\delta}) \\
& + \frac{32\alpha K (L_C+1)\ln T_{\delta}}{\Delta^2_{\min}} \\
& + 8K^2(L_C+1)^2 \ln \frac{6K^2}{\delta}.
\end{align*}
\end{lemma}

\begin{proof}
By Line 9.A of Algorithm \ref{alg:CCB}, as soon as we have $l_{ij} > 0.5$, the set $\mathcal{B}^i_t$ will be emptied. In what follows, we will show that the probability that the number of time-steps before we have $l_{ij} > 0.5$ is greater than 
\begin{align*}
\Delta T := \widehat{N}^{\delta/2}(T_{\delta}) + N 
\end{align*}
with
\[ N := \frac{32\alpha K (L_C+1)\ln T_{\delta}}{\Delta^2_{\min}} + 8K^2(L_C+1)^2 \ln \frac{6K^2}{\delta} \]
is bounded by $\delta/6K^2$. This is done using the amount of exploration infused by Line 10 of Algorithm \ref{alg:CCB}. To begin, let us note that by Corollary \ref{cor:auto-comp}, there is a time-step before $T_0:=C(\delta/2)+\widehat{N}^{\delta/2}(T_{\delta})$ when the condition of Line 9.C of Algorithm \ref{alg:CCB} is satisfied for some Copeland winner. At this point, if $\mathcal{B}^i_t$ contains fewer than $L_C+1$ elements, then it will be emptied; furthermore, for all $k>C$, the sets $B^k_{T_0}$ will have at most $L_C+1$ elements and so the set 
\[ \mathcal{S}_t := \{(k,\ell)|a_\ell\in\mathcal{B}^k_t \textup{ and } 0.5\in[l_{k\ell},u_{k\ell}]\} \]
contains at most $K(L_C+1)$ elements for all $t\geq T_0$. Moreover, if at time-step $T_1 := C(\delta/2)+\Delta T$ we have $a_j \in \mathcal{B}^i_{T_1}$, then we can conclude that $(i,j) \in \mathcal{S}_t$ for all $t \in [C(\delta/2),T_1]$, since, if at any time after $C(\delta/2)$ arm $a_j$ were to be removed from $\mathcal{B}^i_{t}$, it will never be added back because that can only happen through Line 9.B of Algorithm \ref{alg:CCB} and by Lemma \ref{lem:HighProbBound} and the assumption of the lemma we have $u_{ij} > p_{ij} > 0.5$.


What we can conclude from the observations in the last paragraph is that if at time-step $T_1$ we still have $a_j \in \mathcal{B}^i_{T_1}$, then there are $\Delta T$ time-steps during which the probability of comparing arms $a_i$ and $a_j$ was at least $\frac{1}{4K(L_C+1)}$ and yet no more than $\frac{4\alpha\ln T_\delta}{\Delta_{ij}^2}$ comparisons took place, since otherwise, we would have $l_{ij} > 0.5$ at some point before $T_1$. Now, let $B^{ij}_n$ denote the indicator random variable that is equal to $1$ if arms $a_i$ and $a_j$ were chosen to be played against each other by Line 10 of Algorithm \ref{alg:CCB} during time-step $T_1+n$. Also, let $X_1, \ldots, X_N$ be iid Bernoulli random variables with mean $\frac{1}{4K(L_C+1)}$. Since $B^{ij}_n$ and $X_n$ are Bernoulli and we have $\mathbb{E}\left[B^{ij}_n\right] \leq \mathbb{E}[X_n]$ for each $n$, then we can conclude that 
\[ P\left(\sum_{n=1}^N B^{ij}_n < s\right) \leq P\left(\sum_{n=1}^N X_n < s\right) \textup{ for all } s. \]

On the other hand, we can use the Hoeffding bound to show that the right hand side of the above inequality is smaller than $\delta/6$ if we set $s = \frac{4\alpha\ln T_\delta}{\Delta_{ij}^2}$:
\begin{align*}
P\left(\sum_{n=1}^N X_n < \frac{4\alpha\ln T_\delta}{\Delta_{ij}^2} \right) & \leq P\left(\sum_{n=1}^N X_n < \frac{4\alpha\ln T_\delta}{\Delta_{\min}^2} \right) \\
& = P\left(\sum_{n=1}^N X_n < \frac{N}{4K(L_C+1)} - a \right) \leq e^{-\dfrac{2a^2}{N}} \\
& \qquad\qquad \textup{with } a := -\frac{4\alpha\ln T_\delta}{\Delta_{\min}^2} + \frac{N}{4K(L_C+1)} \\
& = e^{ -\frac{32\alpha^2 \ln^2 T_\delta}{\Delta^4_{\min} N} + \frac{4\alpha\ln T_{\delta}}{K(L_C+1)\Delta^2_{\min}} - \frac{N}{8K^2(L_C+1)^2} } \\
& \leq e^{ \frac{4\alpha\ln T_{\delta}}{K(L_C+1)\Delta^2_{\min}} - \frac{N}{8K^2(L_C+1)^2} } \\
& = e^{-\ln 6K^2/\delta} = \delta/6K^2.
\end{align*}

Now, if we take a union bound over all pairs of arms $a_i$ and $a_j$ satisfying the condition stated at the beginning of this scenario, we get that with probability $\delta/6$ by time-step $C(\delta/2)+\Delta T$ all such erroneous hypotheses are reset by Line 9.A of Algorithm \ref{alg:CCB}, emptying the sets $\mathcal{B}^i_t$.
\end{proof}

\begin{lemma}\label{lem:Bt}
Let $t_1 \in [C(\delta/2),T_\delta)$ be such that for all $i,j$ satisfying $a_j \in \mathcal{B}^i_{t_1}$ we have $p_{ij} < 0.5$. Then, the following two statements hold with probability $1-5\delta/6$:
\begin{enumerate}[leftmargin=*]
\item If the set $\mathcal{B}_{t_1}$ in Algorithm \ref{alg:CCB} contains at least one Copeland winner, then if we set $t_2 = t_1 + n_{\max}$, where 
\[ n_{\max} := 2K \max_{i>C}\widehat{N}^{\delta/2}_i(T_\delta) + \frac{K^2\ln (6K/\delta)}{2}, \]
then $\mathcal{B}_{t_2}$ is non-empty and contains no non-Copeland winners, i.e. for all $a_i \in \mathcal{B}_{t_2}$ we have $i \leq C$.
\item If the set $\mathcal{B}_{t_1}$ in Algorithm \ref{alg:CCB} contains no Copeland winners, i.e. for all $a_i \in \mathcal{B}_{t_1}$, we have $i > C$, then within $n_{\max}$ time-steps the set $\mathcal{B}_t$ will be emptied by Line 9.B of Algorithm \ref{alg:CCB}.
\end{enumerate}
Therefore, with probability $1-5\delta/6$, by time $t_1 + 2n_{\max}$ all non-Copeland winners (i.e. arms $a_i$ with $i > C$) are eliminated from $\mathcal{B}_t$.
\end{lemma}

\begin{proof}
We will consider the two cases in the following, conditioning on the conclusions of Lemma \ref{lem:HighProbBound}, Proposition \ref{prop:counts} and Corollary \ref{cor:auto-comp}, all simultaneously holding with $1-\delta/2$:
\begin{enumerate}[leftmargin=*]
\item $\mathcal{B}_{t_1}$ {\bf contains} a Copeland winner (i.e. $a_c \in \mathcal{B}_{t_1}$ for some $c \leq C$): in this case, by Lemma \ref{lem:HighProbBound}, we know that the Copeland winner will forever remain in the set $\mathcal{B}_t$ because 
\[ \overline{\textup{Cpld}}(a_c) \geq \max_j \textup{Cpld}(a_j) \geq \max_j \underline{\textup{Cpld}}(a_j), \]
then $\mathcal{B}_{t_2}$ will indeed be empty. Moreover, in what follows, we will show that the probability that any non-Copeland winner in $\mathcal{B}_t$ is not eliminated by time $t_2$ is less than $\delta/6$. Let us assume by way of contradiction that there exists an arm $a_b$ with $b>C$ such that $a_b$ is in $\mathcal{B}_{t_2}$: we will show that the probability of this happening is less than $\delta/6K$, and so, taking a union bound over non-Copeland winning arms, the probability that any non-Copeland winner is in $\mathcal{B}_{t_2}$ is seen to be smaller than $\delta/6$.

Now, to see that the probability of $a_b$ being in the set $\mathcal{B}_{t_2}$ is small, note that the fact that $a_b$ being in $\mathcal{B}_{t_2}$ implies that $a_b$ was in the set $\mathcal{B}_t$ for the entirety of the time interval $[C(\delta/2),t_2]$ as we will show in the following. If $a_b$ is eliminated from $\mathcal{B}_t$ at some point between $t_1$ and $t_2$, it will not get added back into $\mathcal{B}_t$ because that can only take place if the set $\mathcal{B}_t$ is reset at some point and there are only two ways for that to happen:
\begin{enumerate}[leftmargin=*]
\item By Line 9.A of Algorithm \ref{alg:CCB} in the case that for some pair $(i,j)$ with $a_j \in \mathcal{B}^i_t$ we have $l_{ij} > 0.5$; however, this is ruled out by our assumption that at time $t_1$ we have $p_{ij} < 0.5$ and by Lemma \ref{lem:HighProbBound}, which stipulates that we have $l_{ij} \leq p_{ij} < 0.5$.
\item By Line 9.B of Algorithm \ref{alg:CCB} in the case that all arms are eliminated from $\mathcal{B}_t$, but this cannot happen by the fact mentioned above that $a_c$ will not not be removed from $\mathcal{B}_t$.
\end{enumerate}

So, as mentioned above, we indeed have that at each time-step between $t_1$ and $t_2$, the set $\mathcal{B}_t$ contains $a_b$. Next, we will show that the probability of this happening is less than $\delta/6K$. To do so, let us denote by $\mathcal{S}_b$ the time-steps when arm $a_b$ was in the set of optimistic Copeland winners, i.e.

\[ \mathcal{S}_b := \left\{\; t \in (t_1,t_2] \,\big|\, a_b \in \mathcal{C}_t \;\right\}. \]

We can use Corollary \ref{cor:auto-comp} above with $T = T_\delta$ to show that the size of the set $\mathcal{S}_b$ (which we denote by $|\mathcal{S}_b|$) is bounded from below by $t_2-t_1-\sum_{i \neq j} \widehat{N}^{\delta/2}_{ij}(T_{\delta})$: this is because whenever any Copeland winner $a_c$ is played against itself, Equation \eqref{eqn:topCpld} holds, and so if we were to have $a_b \notin \mathcal{C}_t$ during that time-step $a_b$ would have had to get eliminated from $\mathcal{B}_t$ because $a_b$ not being an optimistic Copeland winner would imply that 
\[ \overline{\textup{Cpld}}(a_b) < \underline{\textup{Cpld}}(a_c) = \overline{\textup{Cpld}}(a_c). \]

But, we know from facts (a) and (b) above that $a_b$ remains in $\mathcal{B}_t$ for all $t \in (t_1,t_2]$. Therefore, as claimed, we have

\begin{align}
|\mathcal{S}_b| & \geq t_2-t_1-\sum_{i \neq j} N^{\delta/2}_{ij}(T_{\delta}) \geq 2K \widehat{N}^{\delta/2}_b(T_{\delta}) + \frac{K^2\ln (6K/\delta)}{2} =: n_b, \label{eqn:denominator}
\end{align}

where the last inequality is due to the definition of $n_{\max} := t_2-t_1$. On the other hand, Proposition \ref{prop:counts} tells us that the number of time-steps between $t_1$ and $t_2$ when $a_b$ could have been chosen as an optimistic Copeland winner is bounded as
\begin{equation}\label{eqn:numerator}
N^{\delta/2}_b(T_\delta) \leq \widehat{N}^{\delta/2}_b(T_\delta).
\end{equation}

Furthermore, given the fact that during each time-step $t \in \mathcal{S}_b$ we have $a_b \in \mathcal{B}_t \cap \mathcal{C}_t$, the probability of $a_b$ being chosen as an optimistic Copeland winner is at least $1/K$ because of the sampling procedure in Lines 14-17 of Algorithm \ref{alg:CCB}.
However, this is considerably higher than the ratio obtained by dividing the right-hand sides of Inequality \eqref{eqn:numerator} by that of Inequality \eqref{eqn:denominator}. We will make this more precise in the following: for each $t \in \mathcal{S}_b$, denote by $\mu^b_t$ the probability that arm $a_b$ would be chosen as the optimistic Copeland winner by Algorithm \ref{alg:CCB}, and let $X^b_t$ be the Bernoulli random variable that returns $1$ when arm $a_b$ is chosen as the optimistic Copeland winner or $0$ otherwise. As pointed out above, we have that $\mu^b_t \geq \frac{1}{K}$ for all $t \in \mathcal{S}_b$, which, together with the fact that $|\mathcal{S}_b| \geq n_b$, implies that the random variable $X^b := \sum_{t \in \mathcal{S}_b} X^b_t$ satisfies

\begin{equation}\label{eqn:BinomBound}
P(X_b < x) \leq  P(Binom(n_b,1/K) < x).
\end{equation}
This is both because the Bernoulli summands of $X_b$ have higher means than the Bernoulli summands of $Binom(n_b,1/K)$ and because $X_b$ is the sum of a larger number of Bernoulli variables, so $X_b$ has more mass away from $0$ than does $Binom(n_b,1/K)$. So, we can bound the right-hand side of Inequality \eqref{eqn:BinomBound} by $\delta/6K$ with $x = \widehat{N}^{\delta/2}_b(T_\delta)$ to get our desired result. But, this is a simple consequence of the Hoeffding bound, a more general form of which is quoted in Section \ref{sec:prelim}. More precisely, we have
\begin{align*}
P\left(Binom(n_b,1/K) < \widehat{N}^{\delta/2}_b(T_\delta) \right) & = P\left( Binom(n_b,1/K) < \frac{n_b}{K} - a  \right) \\
& \qquad\qquad\qquad \textup{ with } a := \frac{n_b}{K} - \widehat{N}^{\delta/2}_b(T_\delta) \\
& \quad < e^{-2a^2/n_b} = e^{^{\frac{-2\left(\frac{n_b}{K} - \widehat{N}^{\delta/2}_b(T_\delta)\right)^2}{n_b}}} \\
& \quad = e^{-2n_b/K^2  + 4\widehat{N}^{\delta/2}_b(T_\delta)/K -2 \widehat{N}^{\delta/2}_b(T_\delta)^2/n_b} \\
& \quad \leq e^{-2n_b/K^2  + 4\widehat{N}^{\delta/2}_b(T_\delta)/K} = e^{-\ln(6K/\delta)} = \delta/6K
\end{align*}

Using the union bound over the non-Copeland winning arms that were in $\mathcal{B}_{t_1}$, of whom there is at most $K-1$, we can conclude that with probability $\delta/6$ they are all eliminated from $\mathcal{B}_{t_2}$.

\item $\mathcal{B}_{t_1}$ {\bf does not contain} any Copeland winners: in this case, we can use the exact same argument as above to conclude that the probability that the set $\mathcal{B}_t$ is non-empty for all $t \in (t_1,t_2]$ is less than $\delta/6$ because as before the probability that each arm $a_b \in \mathcal{B}_{t_1}$ is not eliminated within $n_b$ time-steps is smaller than $\delta/6K$. \qedhere
\end{enumerate}

\end{proof}


Let us now state the following consequence of the previous lemmas:

{\bf Lemma \ref{lem:SetsB}.} \emph{Given $\delta > 0$, the following fact holds with probability $1-\delta$: for each $i>C$, the set $\mathcal{B}^i_{T_\delta}$ contains exactly $L_C+1$ elements with each element $a_j$ satisfying $p_{ij} < 0.5$. Moreover, for all $t \in [T_\delta,T]$, we have $\mathcal{B}^i_t = \mathcal{B}^i_{T_\delta}$.}

\begin{proof}
In the remainder of the proof, we will condition on the high probability event that the conclusions of Lemma \ref{lem:HighProbBound}, Corollary \ref{cor:auto-comp}, Lemma \ref{lem:preBt} and Lemma \ref{lem:Bt} all hold simultaneously with probability $1-\delta$.

Combining Lemma \ref{lem:Bt}, we can conclude that by time-step $T_1 := T_0 + 2n_{\max}$ all non-Copeland winners are removed from $\mathcal{B}_{T_1}$, which also means by Line 9.B of Algorithm \ref{alg:CCB} that the corresponding sets $\mathcal{B}^i_{T_1}$, with $i > C$ are non-empty, and Lemma  \ref{lem:preBt} tells us that these sets have at least $L_C+1$ elements $a_j$ each of which beats $a_i$ (i.e. $p_{ij} < 0.5$).

Now, applying Corollary \ref{cor:auto-comp}, we know that within $\widehat{N}^{\delta/2}(T_\delta)$ time-steps, Line 9.C of Algorithm \ref{alg:CCB} will be executed, at which point we will have $\overline{L}_C = L_C$ and so $\mathcal{B}^i_t$ will be reduced to $L_C+1$ elements. Moreover, by Lemma \ref{lem:HighProbBound}, for all $t > T_1$ and $a_j \in \mathcal{B}^i_t$ we have $l_{ij} \leq p_{ij} < 0.5$ and so $\mathcal{B}^i_t$ will not be emptied by any of the provisions in Line 9 of Algorithm \ref{alg:CCB}.

Now, since by definition we have $T^\delta \geq T_1 + \widehat{N}^{\delta/2}(T_\delta)$, we have the desired result. 
\end{proof}

\newpage

\section{Proof of Lemma \ref{lem:NumNonCopeland}}
\label{sec:NumNonCopeland-proof}

{\bf Lemma \ref{lem:NumNonCopeland}} \emph{Given a Copeland bandit problem satisfying Assumption {\bf A} 
and any $\delta > 0$, with probability $1-\delta$ the following statement holds: the number of time-steps between $T_{\delta/2}$ and $T$ when each non-Copeland winning arm $a_i$ can be chosen as optimistic Copeland winners (i.e. time-steps when arm $a_c$ in Algorithm \ref{alg:CCB} satisfies $c = i>C$) is bounded by}
\[ \widehat{N}^i := 2\widehat{N}^i_{\mathcal{B}} + 2\sqrt{\widehat{N}^i_{\mathcal{B}}} \ln \frac{2K}{\delta}, \]
\emph{where}
\[ \widehat{N}^i_{\mathcal{B}} := \sum_{j \in \mathcal{B}^i_{T_{\delta/2}}} \widehat{N}_{ij}^{\delta/4}(T). \]

\begin{proof}
The idea of the argument is outlined in the following sequence of facts: 
\begin{enumerate}[leftmargin=*,topsep=0mm,parsep=0pt,itemsep=0pt,partopsep=0pt]
\item\label{item:A} By Lemma \ref{lem:SetsB}, we know that with probability $1-\delta/2$, for each $i>C$ and all times $t>T_{\delta/2}$ the sets $\mathcal{B}^i_t$ will consist of exactly $L_C+1$ arms that beat the arm $a_i$, and that $\mathcal{B}^i_t = \mathcal{B}^i_{T_{\delta/2}}$.

\item\label{item:B} Moreover, if at time $t>T_{\delta/2}>C(\delta/4)$, Algorithm \ref{alg:CCB} chooses a non-Copeland winner as an optimistic Copeland winner (i.e. $i>C$), then with probability $1-\delta/4$ we know that 
\[ \overline{\textup{Cpld}}(a_i) \geq \overline{\textup{Cpld}}(a_1) \geq \textup{Cpld}(a_1) = K-1-L_C. \] 

\item\label{item:C} This means that there could be at most $L_C$ arms $a_j$ that optimistically lose to $a_i$ (i.e. $u_{ij} < 0.5$) and so at least one arm $a_b \in \mathcal{B}^i_t$ does satisfy $u_{ib} \geq 0.5$

\item\label{item:D} This, in turn, means that in Line 13 of Algorithm \ref{alg:CCB} with probability $0.5$ the arm $a_d$ will be chosen from $\mathcal{B}^i_t$.

\item\label{item:E} By Proposition \ref{prop:counts}, we know that with probability $1-\delta/4$, in the time interval $[T_{\delta/2},T]$ each arm $a_j \in \mathcal{B}^i_{T_{\delta/2}}$ can be compared against $a_i$ at most $\widehat{N}_{ij}^{\delta/4}(T)$ many times.
\end{enumerate}

Given that by Fact \ref{item:C} above we need at least one arm $a_j \in \mathcal{B}^i_t$ to satisfy $u_{ij} \geq 0.5$ for Algorithm \ref{alg:CCB} to set $(c,d) = (i,j)$, and that by Fact \ref{item:D} arms from $\mathcal{B}^i_t$ have a higher probability of being chosen to be compared against $a_i$, this means that arm $a_i$ will be chosen as optimistic Copeland winner roughly twice as many times we had $(c,d)=(i,j)$ for some $j \in \mathcal{B}^i_{T_{\delta/2}}$. A high probability version of the claim in the last sentence together with Fact \ref{item:E} would give us the bound on regret claimed by the theorem. In the remainder of this proof, we will show that indeed the number of times we have $c=i$ is unlikely to be too many times higher than twice the number of times we get $(c,d)=(i,j)$, where $j \in \mathcal{B}^i_{T_{\delta/2}}$. To do so, we will introduce the following notation:
\begin{description}[topsep=0mm,parsep=0pt,partopsep=0pt]
\item[$N^i$:] the number of time-steps between $T_{\delta/2}$ and $T$ when arm $a_i$ was chosen as optimistic Copeland winner.
\item[$B^i_n$:] the indicator random variable that is equal to 1 if Line 13 in Algorithm \ref{alg:CCB} decided to choose arm $a_d$ only from the set $B^i_{t_n}$ and zero otherwise, where $t_n$ is the $n^{th}$ time-step after $T_{\delta/2}$ when arm $a_i$ was chosen as optimistic Copeland winner. Note that $B^i$ is simply a Bernoulli random variable mean 0.5.
\item[$N^i_{\mathcal{B}}$:] the number of time-steps between $T_\delta$ and $T$ when arm $a_i$ was chosen as optimistic Copeland winner and that Line 13 in Algorithm \ref{alg:CCB} chose to pick an arm from $\mathcal{B}^i_{T_{\delta/2}}$ to be played against $a_i$. Note that this definition implies that we have
\begin{align}
N^i_{\mathcal{B}} = \sum_{n=1}^{N^i} B^i_n. \label{eqn:NiB}
\end{align}
Moreover, by Fact \ref{item:E} above, we know that with probability $1-\delta/4$ we have
\begin{align}
N^i_{\mathcal{B}} \leq \widehat{N}^i_{\mathcal{B}} := \sum_{j \in \mathcal{B}^i_{T_{\delta/2}}} \widehat{N}_{ij}^{\delta/4}(T). \label{eqn:NiBbound}
\end{align}
\end{description}

Now, we will use the above high probability bound on $N^i_{\mathcal{B}}$ to put the following high probability bound on $N^i$: with probability $1-\delta/2$ we have
\[ N^i \leq \widehat{N}^i := 2\widehat{N}^i_{\mathcal{B}} + 2\sqrt{\widehat{N}^i_{\mathcal{B}}} \ln \frac{2K}{\delta}. \]


To do so, let us assume that the we have $N^i > \widehat{N}^i$ and consider the first $\widehat{N}^i$ time-steps after $T_{\delta/2}$ when arm $a_i$ was chosen as optimistic Copeland winner and note that by Equation \eqref{eqn:NiB} we have
\[ \sum_{n=1}^{\widehat{N}^i} B^i_{n} \leq N^i_{\mathcal{B}} \]
and so by Inequality \eqref{eqn:NiBbound} with probability $1-\delta/4$ the left-hand side of the last inequality is bounded by $\widehat{N}^i_{\mathcal{B}}$: let us denote this event with $\mathcal{E}$. On the other hand, if we apply the Hoeffding bound (cf. Appendix \ref{sec:prelim}) to the variables $B^i_1,\ldots,B^i_{\widehat{N}^i}$, we get

\begin{align}
P\left(\mathcal{E} \;\wedge\; N^i > \widehat{N}^i \right) & \leq P\left( \sum_{n=1}^{\widehat{N}^i} B^i_{n} < \widehat{N}^i_{\mathcal{B}} \right) \nonumber \\
& \hspace{-5mm} = P\left( \sum_{n=1}^{\widehat{N}^i} B^i_{n} < \widehat{N}^i/2 - \sqrt{\widehat{N}^i_{\mathcal{B}}} \ln \frac{2K}{\delta} \right) \nonumber \\
& \hspace{-5mm} \leq e^{-\dfrac{\bcancel{2}\widehat{N}^i_{\mathcal{B}} \left(\ln \frac{2K}{\delta}\right)^2}{\bcancel{2}\widehat{N}^i_{\mathcal{B}} + \bcancel{2}\sqrt{\widehat{N}^i_{\mathcal{B}}} \ln \frac{2K}{\delta}}} \label{eqn:CH}
\end{align}
To simplify the last expression in the last chain of inequalities, let us use the notation $\alpha := \widehat{N}^i_{\mathcal{B}}$ and $\beta := \ln \frac{2K}{\delta}$. Given this notation, we claim that the following inequality holds if we have $\alpha \geq 4$ and $\beta \geq 2$ (which hold by the assumptions of the theorem):
\begin{align}
\frac{\alpha \beta^2}{\alpha + \sqrt{\alpha} \beta} \geq \beta. \label{eqn:ab1}
\end{align}
To see this, let us multiply both sides by the denominator of the left-hand side of the above inequality:
\begin{align}
\alpha \beta^2 \geq \alpha \beta + \sqrt{\alpha} \beta. \label{eqn:ab2}
\end{align}
To see why Inequality \eqref{eqn:ab2} holds, let us note that the restrictions imposed on $\alpha$ and $\beta$ imply the following pair of inequalities, whose sum is equivalent to Inequality \eqref{eqn:ab2}:
\begin{equation*}
\begin{array}{rcl}
\alpha \beta^2 & \geq & 2\alpha \beta \\
+ \quad \alpha \beta^2 & \geq & 2\sqrt{\alpha} \beta^2 \\
\midrule
= \; 2\alpha \beta^2 & \geq & 2\alpha \beta + 2\sqrt{\alpha} \beta^2
\end{array}
\end{equation*}
Now that we know that Inequality \eqref{eqn:ab1} holds, we can combine it with Inequality \eqref{eqn:CH} to get
\begin{align*}
P\left(\mathcal{E} \;\wedge\; N^i > \widehat{N}^i \right) \leq e^{-\ln \dfrac{2K}{\delta}} = \frac{\delta}{2K}.
\end{align*}
Taking a union over the non-Copeland winning arms, we get
\[ P(\mathcal{E} \;\wedge\; \forall\,i>C,\, N^i > \widehat{N}^i) > 1-\delta/2. \]
So, given the fact that we have $P(\mathcal{E}) < \delta/4$, we know that with probability $1-\delta$ each non-Copeland winner is selected as optimistic Copeland winner between $T_{\delta/2}$ and $T$ no more than $\widehat{N}^i$ times.
\end{proof}




\section{A Scalable Solution to the Copeland Bandit Problem}
\label{sec:SCBAnalysis}

In this section, we prove Lemma~\ref{lem:apx_cop_main_m}, providing an analysis to the PAC solver of the Copeland winner identification algorithm.

To simplify the proof, we begin by solving a slightly easier variant of Lemma~\ref{lem:apx_cop_main_m} where the queries are deterministic. Specifically, rather than having a query to the pair $(a_i,a_j)$ be an outcome of a Bernoulli r.v.\ with an expected value of $p_{ij}$, we assume that such a query simply yields the answer to whether $p_{ij}>0.5$. Clearly, a solution can be obtained using $K(K-1)/2$ many queries but we aim for a solution with query complexity linear in $K$. In this section we prove the following.

\begin{lemma} \label{lem:apx-cop-det}
Given $K$ arms and a parameter $\eps$, Algorithm~\ref{alg:copeland-apx} finds 
a $(1+\eps)$-approximate best arm with probability at least $1-\delta$, by using at most 
$$ \log(K/\delta) \cdot {\cal O} \left(  K \log(K) + \min\left\{ \frac{K}{\eps^2}, K^2(1-\cop(a_{1})) \right\} \right)  $$
many queries.
In particular, when there is a Condorcet winner ($\cop(a_{1})=1$) or more generally $\cop(a_{1}) = 1-\mathcal{O}(1/K)$, an exact solution can be found with probability at least $1-\delta$ by using at most
$$  {\cal O} \left(  K \log(K) \log(K/\delta) \right) $$
many queries.
\end{lemma}

The idea behind our algorithm is as follows. We provide an unbiased estimator of the normalized Copeland score of arm $a_i$ by picking an arm $a_j$ uniformly at random and querying the pair $(a_i,a_j)$. This method allows us to apply proof techniques for the classic MAB problem. These techniques provide a bound on the number of queries dependent on the gaps between the different Copeland scores. Our result is obtained by noticing that there cannot be too many arms with a large Copeland score; the formal statement is given later in Lemma~\ref{lem:copeland obs}. If the Copeland winner has a large Copeland score, i.e., $L_C$ is small, then only a small number of arms can be close to optimal. Hence, the main argument of the proof is that the majority of arms can be eliminated quickly and only a handful of arms must be queried many times.

As stated above, our algorithm uses as a black box Algorithm~\ref{alg:kl-MAB}, an approximate-best-arm identification algorithm for the classical MAB setup. Recall that here, each arm $a_i$ has an associated reward $\mu_i$ and the objective is to identify an arm with the (approximately) largest reward. Without loss of geenrality, we assume that $\mu_1$ is the maximal reward. The following lemma provides an analysis of Algorithm~\ref{alg:kl-MAB} that is tight for the case where $\mu_{1}$ is close to 1. In this case, it is exactly the set of near optimal arms that will be queried many times hence it is important to take into consideration that the random variables associated with near optimal arms have a variance of roughly $1-\mu_i$, which can be quite small. This translates to savings in the number of queries to arm $a_{i}$ by a factor of $1-\mu_i$ compared to an algorithm that does not take the variances into account. 

\begin{lemma} \label{lem:kl-mab}
Algorithm~\ref{alg:kl-MAB} requires as input an error parameter $\eps$, failure probability $\delta$ and an oracle to $k$ Bernoulli distributions. It outputs, with probability at least $1-\delta$, a $(1+\eps)$-approximate best arm, that is an arm $a_{i}$ with corresponding expected reward of $\mu \geq 1-(1-\mu_1)(1+\eps)$ with $\mu_1$ being the maximum expected value among arms. The expected number of queries made by the algorithm is upper bounded by
$$ {\cal O} \left(  \sum_i \frac{(1-\mu_i) \log(K/(\delta \Delta_i \eps))}{\left(\Delta_i^\eps\right)^2}  \right), $$
with $\Delta_i^\eps  = \max \left\{ \mu_{1} - \mu_i, \eps(1-\mu_{1}) \right\}$.
Moreover, with probability\ at least $1-\delta$, the number of times arm  $i$ will be queried is at most 
$$  {\cal O} \left(   \frac{(1-\mu_i)\log(K/(\delta \Delta_i \eps))}{\left(\Delta_i^\eps\right)^2}  \right) \ .$$
\end{lemma}

We prove Lemma~\ref{lem:kl-mab} in Appendix~\ref{sec:KL analysis}.

For convenience, we denote by $\mu_i$ the normalized Copeland score of arm $a_{i}$ and $\mu_1$ the maximal normalized Copeland score. To get an informative translation of the above expression to our setting, let $A$ be the set of arms with normalized Copeland score in $(1-2(1-\mu_1),\mu_1]$ and let $\bar{A}$ be the set of the other arms. In our setting, this query complexity of Algorithm~\ref{alg:kl-MAB} is upper bounded by 
\begin{equation} \label{eq:query-det}
{\cal O} \left(   \frac{2|A|\log(K/\delta)}{(1-\mu_1)\eps^2}  + \sum_{i \in \bar{A}} \frac{\log(K/\delta)(1-\mu_i)}{(\mu_1-\mu_i)^2} \right),
\end{equation}
assuming\footnote{The value of $\delta$ we require is $1/T$. If the assumption does not follow in that case, the regret must be linear and all of the statements hold trivially.} $\delta < (1-\mu_1)\eps$.

It remains to provide an upper bound for the above expression given the structure of the normalized Copeland scores. In particular, we use the results of Lemma~\ref{lem:copeland obs}, repeated here for convenience.

{\bf Lemma \ref{lem:copeland obs}.} \emph{Let $D \subset [K]$ be the set of arms for which $\cop(a_{i}) \geq 1-d/(K-1)$, that is arms that are beaten by at most $d$ arms. Then $|D| \leq 2d+1$.}

We bound the left summand in \eqref{eq:query-det}: 
\begin{equation} \label{eq:lhs}
\frac{2|A|\log(K/\delta)}{(1-\mu_1)\eps^2} \leq \frac{\left(4(1-\mu_1)(K-1)+2 \right) \log(K/\delta)}{(1-\mu_1)\eps^2} = O\left( \frac{\log(K/\delta) K}{\eps^2} \right).
\end{equation}

We now bound the right summand in \eqref{eq:query-det}. Let $i \in \bar{A}$. According to the definition of $\bar{A}$ it holds that $(1-\mu_i) \leq 2(\mu_1-\mu_i)$.
Hence:
$$ \sum_{i \in \bar{A}} \frac{\log(K/\delta)(1-\mu_i)}{(\mu_1-\mu_i)^2} \leq \sum_{i \in \bar{A}} \frac{4\log(K/\delta)}{1-\mu_i}.  $$

\begin{lemma} \label{lem:rhs} We have $\displaystyle\sum_{i  :\ \mu_i < 1} \frac{1}{1-\mu_i} = \mathcal{O}(K\log(K))$.
\end{lemma}

\begin{proof}
Let $A_{\tau}$ be the set of arms for which $2^{\tau} \leq 1-\mu_i  < 2^{\tau+1}$. According to Lemma~\ref{lem:copeland obs}, we have that $|A_{\tau}| \leq 2^{\tau+2}(K-1)+1$. Other than that, since $1\geq 1-\mu_i \geq 1/(K-1)$ for all $i > C$ we have that $A_\tau = \emptyset$ for any $\tau \leq -\log_2(K-1)-1$ and $\tau > 0$. It follows that:
\begin{align*}
\sum_{i > C} \frac{1}{1-\mu_i} \leq \sum_{\ell=0}^{\lceil \log_2(K-1)\rceil} \frac{|A_{\ell -\log_2(K-1)}|}{2^{\ell -\log_2(K-1)}} & \leq 
\sum_{\ell=0}^{\lceil \log_2(K-1)\rceil} \frac{2^{2+\ell}+1}{2^{\ell -\log_2(K-1)}} \\
& \leq \left( \lceil \log_2(K-1)\rceil + 1\right) \cdot 5(K-1). \qedhere
\end{align*}
\end{proof}

From \eqref{eq:query-det}, \eqref{eq:lhs} and Lemma~\ref{lem:rhs}, we conclude that the total number of queries is bounded by
$$ {\cal O}\left(\log(K/\delta) \left( K \log(K) + \frac{K}{\eps^2}\right)\right). $$
In order to prove Lemma~\ref{lem:apx-cop-det}, it remains to analyze the case where $\eps$ is extremely small. Specifically, when $\eps^2(1-\mu_1)$ takes a value smaller than $1/K$ then the algorithm becomes inefficient in the sense that it queries the same pair more than once. This can be avoided by taking the samples of $j$ when querying the score of arm $a_{i}$ to be uniformly random \emph{without} replacement. The same arguments hold but are more complex as now the arm pulls are not i.i.d. Nevertheless, the required concentration bounds still hold. The resulting argument is that the number of queries is $ \tilde{O}\left(\log(1/\delta) \left(K + \frac{K}{\bar{\eps}^2}\right)\right) $ with $\bar{\eps} = \max\{\eps,1/\left(\sqrt{K(1-\mu_1)}\right)\}$. Lemma~\ref{lem:apx-cop-det} immediately follows.

We are now ready to analyze the stochastic setting. 
\begin{proof} [Proof of Lemma~\ref{lem:apx_cop_main_m}]
By querying arm $a_{i}$ we choose a random arm $j \neq i$ and in fact query the pair $(a_i,a_j)$ sufficiently many times in order to determine whether $p_{ij}>0.5$ with probability at least $1-\delta/K^2$. Standard concentration bounds show that achieving this requires querying the pair $(a_i,a_j)$ at most ${\cal O}\left( \log(K/(\Delta_{ij}\delta)) \Delta_{ij}^{-2}\right)$ many times. It follows that a single query to arm $a_{i}$ in the deterministic case translates into an expected number of 
$${\cal O}\left(\log(KH_i/\delta)) \frac{H_i}{K-1} \right) = {\cal O}\left(\frac{ \log(KH_{\infty}/\delta) H_{\infty}}{K}  \right) $$
many queries in the stochastic setting. The claim now follows from the bound on the expected number of queries given in Lemma~\ref{lem:apx-cop-det}.
\end{proof}

\section{KL-based approximate best arm identification algorithm} \label{sec:KL analysis}
Algorithm~\ref{alg:kl-MAB} solves an approximate best arm identification problem using confidence bounds based on Chernoff's inequality stated w.r.t the KL-divergence of two random variables.
Recall that for two Bernoulli random variables with parameters $p,q$ the KL-divergence from $q$ to $p$ is defined as $d(p,q) = (1-p) \ln((1-p)/(1-q)) + p \ln(p/q)$ with $0\ln(0)=0$. The building block of Algorithm~\ref{alg:kl-MAB} is the well known Chernoff bound stating that for a Bernoulli random variable with expected value $q$, the probability of the average of $n$ i.i.d samples from it to be smaller (larger) than $p$, for $p<q$ ($p>q$), is bounded by $\exp(-n d(p,q))$. 

\begin{algorithm}[h]
\caption{KL-best arm identification}
\label{alg:kl-MAB}
\begin{algorithmic}[1]
{
\REQUIRE Access to oracle giving a noisy approximation of the reward of arm $i$ for $K$ arms, success probability $\delta>0$, approximation parameter $\eps>0$
\FORALL{$i \in [K]$} 
	\STATE $T = 1$
	\STATE $S_i \gets \text{reward}(i)$
	\STATE $I_i \gets [0,1]$
\ENDFOR
\STATE $B \gets [K]$
\STATE $t \gets 2$
\WHILE{$\frac{1-\max_{i \in B} \min I_i}{1- \max_{i \in B} \max I_i}>(1+\eps)$}
	\STATE For all $i \in B$, $S_i \gets S_i+\text{reward}(i)$
	\STATE For all $i \in B$, let $I_i = \{q \in [0,1], \ t \cdot  d(\frac{S_i}{t}, q) \leq \ln(4tK/\delta) + 2\ln \ln(t) \}$
	\STATE For all $i \in B$ for which there exist some $j \in B$ with $\max \{q \in I_i\} < \min \{q \in I_j\}$, remove $i$ from $B$.
	\STATE $t \gets t+1$
\ENDWHILE

\ENSURE  $\arg \max_{i \in B} \min I_i$.
}
\end{algorithmic}
\end{algorithm}

\begin{proof} [Proof of Lemma~\ref{lem:kl-mab}]

We use an immediate application of the Chernoff-Hoeffding bound
\begin{lemma}
Fix $i \in [K]$. Let $E_t^i$ denote the event that at iteration $t$, $\mu_i \notin I_i$. We have that 
$\Pr[E_t^i] \leq 2\frac{\delta}{4tK} \cdot \frac{1}{ \log\left( t \right)^{2}} \leq \frac{\delta}{2 t \log(t)^2 K}$.

\end{lemma}

Let $E$ denote the union, over all $t,i$ of events $E_t^i$. That is, $E$ denotes the event in which there exist some iteration $t$, and for some arm $a_{i}$ such that  $\mu_i \notin I_i$. By the above lemma we get that 
$$\Pr[E] \leq \sum_t \sum_i \Pr[E_t^i] \leq K \sum_{t=2}^\infty  \frac{\delta}{2 t \log(t)^2 K} \leq \delta $$
It follows that given that event $E$ did not happen,  the algorithm will never eliminate the top arm and furthermore, will output an $(1+\eps)$-approximate best arm.
We proceed to analyze the total number of pulls per arm, while having a separate analysis for $(1+\eps)$-approximate best arms and the other arms. We begin by stating an auxiliary lemma giving explicit bounds for the confidence regions.

\begin{lemma} \label{lem:conf bound}
Assume that event $E$ did not occur and let $\rho \geq 0$. For a sufficiently large universal constant $c$ we have for any $t \geq \frac{c\log(tK/\delta)(1-\mu_i)}{\rho^2}$ that $\max I_i < \mu_i+\rho$.  Also, for $t \geq \frac{c\log(tK/\delta)(1-\mu_i+\rho/2)}{\rho^2}$ it holds that $\min I_i > \mu - \rho$.
\end{lemma}
\begin{proof}
We consider the Taylor series associated with $f(x) = d(p+x,p)$. Since $f(0)=f'(0)=0$ it holds that for any $x \leq 1-p$ there exists some $|x'| \leq |x|$ with
$$f(x) = x^2f''(x') = \frac{x^2}{(p+x')(1-p-x')} \leq \frac{2x^2}{1-p}$$

To prove that $\max I_i < \mu_i+\rho$ we apply the above observation for $\rho \leq 1-\mu_i$ (otherwise $\mu_i+\rho>1$ and the claim is trivial) and reach the conclusion that for sufficiently large universal constant $c$ it holds that 
$$t \cdot d(\mu_i+\rho/2, \mu_i) > \log(tK/\delta) + 2\log \log(tK/\delta)$$
$$t \cdot d(\mu_i+\rho/2, \mu_i+\rho) > \log(tK/\delta) + 2\log \log(tK/\delta)$$
The first inequality dictates that $S_i/t \leq \mu_i+\rho/2$. The second inequality dictates that $t \cdot d(S_i/t, \mu_i+\rho) \geq d(\mu_i+\rho/2, \mu_i+\rho)$ is too large in order for $\mu_i+\rho$ to be an element of $I_i$.

The bound for $\min I_i$ is analogous. Since now we have $t \geq \frac{c\log(tK/\delta)(1-\mu_i+\rho/2)}{\rho^2}$, it holds that
$$t \cdot d(\mu_i-\rho/2, \mu_i) > \log(tK/\delta) + 2\log \log(tK/\delta)$$
$$t \cdot d(\mu_i-\rho/2, \mu_i-\rho) > \log(tK/\delta) + 2\log \log(tK/\delta)$$
This means that first, $S_i/t \geq \mu_i - \rho/2$ and second, that $t \cdot d(S_i/t, \mu_i-\rho) \geq d(\mu_i-\rho/2, \mu_i-\rho)$ is too large in order for $\mu_i-\rho$ to be an element of $I_i$.
\end{proof}

\begin{lemma}
Let $i$ be a suboptimal arm, meaning one where $\mu_i \leq 1-(1-\mu_{1})(1+\eps)$. Denote by $\Delta_i$ its gap $\mu_{1}-\mu_i$. If event $E$ does not occur then $i$ is queried at most $O\left(\frac{\log\left(\frac{K}{\delta \Delta_i}\right)v_i}{(\Delta_i)^2} \right)$ many times, where $v_i = 1-\mu_i$
\end{lemma}
\begin{proof}
We first notice that as we are assuming that event $E$ did not happen, it must be the case that arm $1$ is never eliminated from $B$.  Consider an iteration $t$ such that
\begin{equation} \label{eq:t kl lb subopt}
t \geq \frac{c \log(tK/\delta)v_i}{(\Delta_i)^2}
\end{equation}
for a sufficiently large $c$, then according to Lemma~\ref{lem:conf bound} it holds that $\max I_i < \mu_i + \Delta_i/2$. Now, since $v_i=1-\mu_i \geq 1-\mu_{1}+\Delta_i/2$ we have that for the same $t$ it must be the case that $\min I_{1} > \mu_{1}-\Delta_i/2$. It follows that $\min I_{1} > \max I_i$ and arm $a_{i}$ is eliminated at round $t$.
\end{proof}

\begin{lemma}
Assume $\eps \leq 1$. If event $E$ does not occur then for some sufficiently large universal constant $c$ it holds that when  $t \geq \frac{c\log(tK/\delta)}{(1-\mu_{1})\eps^2}$ the algorithm terminates.
\end{lemma}
\begin{proof}
Let $i$ be an arbitrary arm. Since
$$t \geq \frac{c\log(tK/\delta)}{(1-\mu_{1})\eps^2} = \frac{c\log(tK/\delta)(1-\mu_i)}{(1-\mu_{1})(1-\mu_i)\eps^2} $$
we get, according to Lemma~\ref{lem:conf bound} that
$$ \max I_i \leq \mu_i + \frac{\eps}{3}\sqrt{(1-\mu_i)(1-\mu_1)} $$
In order to bound $\sqrt{(1-\mu_i)(1-\mu_1)}$ we consider the function $f(x) = \sqrt{v(v+x)}$. Notice that $f(0)=v$ and $f'(x) = \frac{v}{2\sqrt{v(v+x)}} \leq \frac{1}{2}$ for $x \geq 0$. It follows that for positive $x$, $\sqrt{v(v+x)} \leq v+x/2$, meaning that
$$ \max I_i \leq \mu_i + \frac{\eps\left((1-\mu_i)+\Delta_i/2\right)}{3} \leq \mu_{1}+\frac{\eps(1-\mu_{1})}{3}$$

Now, since $\eps \leq 1$ we have 
$$t \geq \frac{c\log(tK/\delta)(1-\mu_{1})}{(1-\mu_{1})^2\eps^2} \geq \frac{(c/2)\log(tK/\delta)(1-\mu_{1} + \eps(1-\mu_{1}))}{(1-\mu_{1})^2\eps^2} $$
hence for sufficiently large $c$ we can apply Lemma~\ref{lem:conf bound} and obtain
$$ \min I_{1} \geq \mu_{1} - \frac{\eps(1-\mu_1)}{3} $$
It follows that assuming $\eps \leq 1$,
$$ \min I_{1} \geq 1-\left( 1-\max_i I_i \right)(1+\eps) $$
meaning that the algorithm will terminate at iteration $t$.
\end{proof}

This concludes the proof of Lemma~\ref{lem:kl-mab}
\end{proof}

\newpage

\clearpage

\begin{table}[!t]


\small
    \caption{List of notation used in this paper}
    \label{tbl:notation1}
    \begin{tabularx}{\columnwidth}{@{~}l@{~~}|X@{~}}
	 \toprule
	 Symbol & Definition  \\[5pt]
	 \midrule
	 $K$ & Number of arms \\[5pt]
	 $[K]$ & The set $\{1,\ldots,K\}$ \\[5pt]
	 $a_1,\ldots,a_K$ & Set of arms \\[5pt]
	 $p_{ij}$ & Probability of arm $a_i$ beating arm $a_j$ \\[5pt]
	 $\textup{Cpld}(a_i)$ & Copeland score: number of arms that $a_i$ beats, i.e. $|\{j\,|\,p_{ij} > 0.5\}|$ \\[5pt]
	 $\cop(a_i)$ & Normalized Copeland score: $\dfrac{\textup{Cpld}(a_i)}{K-1}$ \\[5pt]
	 $C$ & Number of Copeland winners, i.e. arms $a_i$ with $\textup{Cpld}(a_i) \geq \textup{Cpld}(a_j)$ for all $j$ \\[5pt]
	 $a_1,\ldots,a_C$ & Copeland winner arms \\[5pt]
	 $\alpha$ & UCB parameter of Algorithm \ref{alg:CCB} \\[5pt]
	 $\delta$ & Probability of failure \\[5pt]
	 $C(\delta)$ & $\left(\dfrac{(4\alpha-1)K^2}{(2\alpha-1)\delta}\right)^{\frac{1}{2\alpha-1}}$ \\[15pt]
	 $N_i(t)$ & Number of times arm $a_i$ was chosen as the optimistic Copeland winner until time $t$ \\[5pt]
	 $N^{\delta}_i(t)$ & Number of times arm $a_i$ was chosen as the optimistic Copeland winner in the interval $(C(\delta),t]$ \\[5pt]
	 $N_{ij}(t)$ & Total number of time-steps before $t$ when $a_i$ was compared against $a_j$ (notice that this definition is symmetric with respect to $i$ and $j$) \\[5pt]
	 $N_{ij}^{\delta}(t)$ & Number of time-steps between times $C(\delta)$ and $t$ when $a_i$ was chosen as the optimistic Copeland winner and $a_j$ as the challenger (note that, unlike $N_{ij}(t)$, this definition is not symmetric with respect to $i$ and $j$) \\[5pt]
	 $\tau_{ij}$ & The last time-step when $a_i$ was chosen as the optimistic Copeland winner and $a_j$ as the challenger (note that $\tau_{ij} \geq C(\delta)$ iff $N_{ij}^{\delta}(t) > 0$) \\[5pt]
	 $w_{ij}(t)$ & Number of wins of $a_i$ over $a_j$ until time $t$  \\[5pt]
	 $u_{ij}(t)$ & $\dfrac{w_{ij}(t)}{N_{ij}(t)} + \sqrt{\dfrac{\alpha\ln t}{N_{ij}(t)}}$ \\[15pt]
	 $l_{ij}(t)$ & $1-u_{ji}(t)$ \\[5pt]
	 $\overline{\textup{Cpld}}(a_i)$ & $\#\left\{k \,|\, u_{ik} \geq \frac{1}{2}, k \neq i\right\}$ \\[5pt]
	 $\underline{\textup{Cpld}}(a_i)$ & $\#\left\{k \,|\, l_{ik} \geq \frac{1}{2}, k \neq i\right\}$ \\[5pt]
	 $\mathcal{C}_t$ & $\{i \,|\, \overline{\textup{Cpld}}(a_i) = \max_j \overline{\textup{Cpld}}(a_j) \}$ \\[5pt]
	 $\mathcal{L}_i$ & the set of arms to which $a_i$ loses, i.e. $a_j$ such that $p_{ij} < 0.5$ \\[5pt]
	 $L_C$ & The largest number of losses that any Copeland winner has, i.e. $\max_{i=1}^C |\{j\, |\, p_{ij} < 0.5\}|$ \\[5pt]
	 $\overline{L}_C$ & Algorithm \ref{alg:CCB}'s estimate of $L_C$ \\[5pt]
	 $\mathcal{B}_t$ & The potentially best arms at time $t$, i.e. the set of arms that according to Algorithm \ref{alg:CCB} have some chance of being Copeland winners \\[5pt]
	 $\mathcal{B}^i_t$ & The arms that at time $t$ have the best chance of beating arm $a_i$ (Cf. Line 12 in Algorithm \ref{alg:CCB}) \\[5pt]
	 $\Delta_{ij}$ & $|p_{ij}-0.5|$ \\[5pt]
	 $\Delta_{\min}$ & $\min \{ \Delta_{ij} | \Delta_{ij} \neq 0 \}$ \\[5pt]
	 $i^*$ & the index of the $(L_C+1)^{th}$ largest element in the set $\{ \Delta_{ij} \,|\, p_{ij} < 0.5 \}$ in the case that $i > C$ \\[5pt]
	 $\Delta^*_i$ & $\left\{\begin{array}{ll} \Delta_{ii^*} & \textup{if } i>C \\[3pt] 0 & \textup{otherwise} \end{array}\right.$ \\[15pt]
	 \bottomrule
     \end{tabularx}


\end{table}

\begin{table}[!t]


\small
    \caption{List of notation used in this paper (Cont'd)}
    \label{tbl:notation2}
    \begin{tabularx}{\columnwidth}{@{~}l@{~~}|X@{~}}
	 \toprule
	 Symbol & Definition  \\[5pt]
	 \midrule
	 $\Delta^*_{ij}$ & $\left\{\begin{array}{ll} \Delta^*_i + \Delta_{ij} & \textup{if } p_{ij} \geq 0.5 \\[3pt] \max\{\Delta^*_i,\Delta_{ij}\} & \textup{otherwise} \end{array}\right.$ \\[10pt]
	 & (See Figures \ref{fig:nonCopelandFigure} and \ref{fig:CopelandNonCopelandFigure} for a pictorial explanation.) \\[5pt]
	 $\Delta^*_{\min}$ & $\displaystyle\min_{i > C} \Delta^*_i$ \\[10pt]
	 $\widehat{N}_{ij}^{\delta}(T)$ & $\left\{\begin{array}{cl} \frac{4\alpha\ln T}{\left(\Delta^*_{ij}\right)^2} & \textup{if } i \neq j \\[10pt] 0 & \textup{if $i=j$ and $i > C$} \end{array}\right.$ \\[15pt]
	 $\widehat{N}_i^{\delta}(T)$ & $\displaystyle\sum_{j=1}^K \widehat{N}_{ij}^{\delta}(T) $ \\[15pt]
	 $\widehat{N}^{\delta}(T)$ & $\displaystyle\sum_{i \neq j} \widehat{N}_{ij}^{\delta}(T)+1$ \\[15pt]
	 $T_\delta \geq$ & $C(\frac{\delta}{2})+8K^2(L_C+1)^2\ln\frac{6K^2}{\delta}+K^2\ln\frac{6K}{\delta}$ \\[5pt]
	 & \hspace{10mm} $+ \frac{32\alpha K (L_C+1)}{\Delta^2_{\min}}\ln T_{\delta}+ \widehat{N}^{\delta/2}(T_{\delta})$ \\[5pt]
	 & \hspace{10mm} $+ 4K \max_{i>C}\widehat{N}^{\delta/2}_i(T_\delta)$ \\[5pt]
	 & $T_\delta$ is the smallest integer satisfying the above inequality (Cf. Definition \ref{def:Tdelta}). \\[5pt]
	 $T_0$ & $C(\delta/2) + \widehat{N}^{\delta/2}(T_{\delta})$ \\[5pt]
	 & \hspace{10mm} $+ \frac{32\alpha K (L_C+1)\ln T_{\delta}}{\Delta^2_{\min}}$ \\[5pt]
	 & \hspace{10mm} $+ 8K^2(L_C+1)^2 \ln \frac{6K^2}{\delta}$ \\[5pt]
	 $n_b$ & $2K \widehat{N}^{\delta/2}_b(\widehat{T}_{\delta}) + \frac{K^2\ln (4K/\delta)}{2}$ \\[5pt]
	 $Binom(n,p)$ & A ``binomial'' random variable obtained from the sum of $n$ independent Bernoulli random variables, each of which produces $1$ with probability $p$ and $0$ otherwise. \\[5pt]
	 $\Delta_i$ & $\max\left\{\cop(a_1)-\cop(a_{i}),\frac{1}{K-1}\right\}$ \\[5pt]
	 $H_i$ & $\displaystyle\sum_{j \neq i} \frac{1}{\Delta_{ij}^2}$ \\[15pt]
	 $H_{\infty}$ & $\max_i H_i$ \\[5pt]
	 $\Delta_i^\eps$ & $\max \left\{ \Delta_i, \eps(1-\cop(a_1)) \right\}$ \\[5pt]
	 \bottomrule
     \end{tabularx}




\end{table}


\end{document}